\documentclass{article}

\usepackage{microtype}
\usepackage{subfigure}
\usepackage{booktabs} 

\usepackage{hyperref}
\usepackage{authblk}
\usepackage{tikz-cd}
\usepackage{tabularx}
\usepackage{multirow}
\usepackage[normalem]{ulem}

\usepackage[accepted]{icml2023}

\usepackage{amsmath}
\usepackage{amssymb}
\usepackage{mathtools}
\usepackage{amsthm}

\usepackage[capitalize,noabbrev]{cleveref}

\theoremstyle{plain}
\newtheorem{theorem}{Theorem}[section]
\newtheorem{proposition}[theorem]{Proposition}
\newtheorem{lemma}[theorem]{Lemma}

\theoremstyle{definition}

\theoremstyle{remark}
\newtheorem{remark}[theorem]{Remark}

\newcommand{\h}{\mathbb{H}}
\newcommand{\D}{\mathbb{D}}
\newcommand{\R}{\mathbb{R}}
\newcommand{\LL}{\mathbb{L}}
\newcommand{\eps}{\varepsilon}

\newcommand{\norm}[1]{\left\|#1\right\|}

\usepackage[textsize=tiny]{todonotes}

\icmltitlerunning{The Numerical Stability of Hyperbolic Representation Learning}

\begin{document}

\twocolumn[
\icmltitle{The Numerical Stability of Hyperbolic Representation Learning}



\icmlsetsymbol{equal}{*}

\begin{icmlauthorlist}
\icmlauthor{Gal Mishne}{hdsi}
\icmlauthor{Zhengchao Wan}{hdsi}
\icmlauthor{Yusu Wang}{hdsi}
\icmlauthor{Sheng Yang}{seas}
\end{icmlauthorlist}

\icmlaffiliation{hdsi}{Hal{\i}c{\i}o\u{g}lu Data Science Institute, University of California San Diego, La Jolla, California, USA}
\icmlaffiliation{seas}{Harvard John A. Paulson School of Engineering and Applied Science, Harvard University, Cambridge, Massachusetts, USA}

\icmlcorrespondingauthor{Zhengchao Wan}{z5wan@ucsd.edu}

\icmlkeywords{Machine Learning, ICML}

\vskip 0.3in
]



\printAffiliationsAndNotice{}  

\begin{abstract}
The hyperbolic space is widely used for representing hierarchical datasets due to its ability to embed trees with small distortion. However, this property comes at a price of numerical instability such that training hyperbolic learning models will sometimes lead to catastrophic NaN problems, encountering unrepresentable values in floating point arithmetic. In this work, we analyze the limitations of two popular models for the hyperbolic space, namely, the Poincar\'e ball and the Lorentz model. We find that, under the 64 bit arithmetic system, the Poincar\'e ball has a relatively larger capacity than the Lorentz model for correctly representing points. However, the Lorentz model is superior to the Poincar\'e ball from the perspective of optimization, which we theoretically validate.
To address these limitations, we identify one Euclidean parametrization of the hyperbolic space which can alleviate these issues. We further extend this Euclidean parametrization to hyperbolic hyperplanes and demonstrate its effectiveness in improving the performance of hyperbolic SVM.

\end{abstract}

\section{Introduction}
The $n$-dimensional hyperbolic space $\mathbb{H}^n$ is the unique simply-connected Riemannian manifold with a constant sectional curvature $-1$. A remarkable property of $\mathbb{H}^n$ against the Euclidean space $\mathbb{R}^n$ is that the volume of a ball in $\mathbb{H}^n$ grows exponentially with respect to the radius. One can thus embed finite trees into $\mathbb{H}^n$ with arbitrarily small distortion \citep{sarkar2011low}. This motivates the study of representation learning of hierarchical data into hyperbolic space \citep{nickel2017poincare} and, moreover, the design of deep neural networks in hyperbolic spaces \citep{ganea2018hyperbolic,chen2022fully}, with applications in various domains where hierarchical data is abundant, such as NLP \citep{zhu2020hypertext,lopez2019fine,lopez2020fully}, recommendation systems \citep{chamberlain2019scalable} and neuroscience \citep{gao2020poincare}.

However, the exponential volume growth property leads to numerical instability in training hyperbolic learning models.
As shown by \cite{sala2018representation}, in order to represent points in the popular Poincar\'{e} model \citep{nickel2017poincare,ganea2018hyperbolic}, one requires a large number of bits to avoid undesirable rounding errors when dealing with small numbers. 
The Lorentz model, a popular alternative for representing the hyperbolic space \citep{nickel2018learning,law2019lorentzian}, suffers an opposite numerical issue in dealing with large numbers.
\cite{yu2019numerically} proved that representing Lorentz points using floating number arithmetic could lead to a huge representation error when the points are far from the origin. 
Besides, in terms of optimization and training, the empirical superiority of the Lorentz model over the Poincar\'e model has been often observed in tasks such as word embeddings \citep{law2019lorentzian,nickel2018learning}, however, the reason behind is yet unclear. 

It is thus important to theoretically clarify the practical limitations of current hyperbolic models and to explain why certain models have better empirical performance than others. The goal of this paper is to address these issues and to provide a new model with better numerical performance.


\paragraph{Our Contributions} (1) We are the first to clearly identify the representation capacity of two popular hyperbolic models, the Lorentz and the Poincar\'{e} models, in a geometric lens under the setting of 64-bit system\footnote{Standard CPU/GPU processors have a max precision of 64 bits, and surpassing this limit requires ticklish custom implementations.}.
(2) We provide a theoretical study of optimization on hyperbolic spaces from the perspective of the gradient vanishing issue inherent to hyperbolic models.
We confirm theoretically that the Lorentz model is better than the Poincar\'e model from this aspect, which was previously only an empirical observation.
(3) We study a simple yet effective Euclidean parameterization of the hyperbolic space which turns out to alleviate the numerical issues present in the Lorentz and Poincar\'{e} models. Although this Euclidean parametrization has been utilized in the literature (see for example \cite{mathieu2019continuous}), we provide new insights into this construction by theoretically showing that the Euclidean parameterization is as good as the Lorentz model in terms of optimization; but it has no limitation in representation capacity.
We also empirically validate that the Euclidean parameterization improves the performance in tree embedding tasks compared to both the Lorentz and Poincar\'{e} models.
Moreover, we apply our Euclidean parametrization to hyperbolic SVM and propose the method LSVMPP, which has superior performance in comparison to other methods in our experiments. 


\section{Preliminary}\label{sec:pre}
In this section, we provide the necessary background for the $n$-dimensional hyperbolic space $\mathbb{H}^n$. There are multiple isometric models for describing $\h^n$ (see \citep{peng2021hyperbolic} for a detailed survey). In this paper, we mainly focus on the Poincar\'e ball and the Lorentz model.

\paragraph{Poincar\'{e} Ball} The Poincar\'{e} ball is the unit open ball $\mathbb{D}^n\subseteq\R^n$ with a Riemannian metric conformal to the underlying Euclidean one: at each point $x\in \mathbb{D}^n$, the metric $g_x=\lambda_x^2h_x$, where $h_x$ stands for the Euclidean metric at $x$ and $\lambda_x:=\frac{2}{(1-\|x\|^2)}$. For any $x,y\in \mathbb{D}^n$, the geodesic distance between them is expressed as
\begin{equation}\label{eq:poincare distance}
    d_{\mathbb{D}^n}(x,y)=\mathrm{arccosh}\left(1+ \frac{2\norm{x-y}^2}{(1-\norm{x}^2)(1-\norm{y}^2)}\right)
\end{equation}

\paragraph{Lorentz Model} The Lorentz model interprets $\h^n$ as a submanifold of the so-called Minkowski space $\R^{n,1}$. This is the linear space $\R^{n+1}$ equipped with the Minkowski product $[x,y]:=-x_0y_0+\sum_{i=1}^nx_iy_i$. The Lorentz model (also called the hyperboloid) is then the $n$-dim ``unit sphere'' in the Minkowski space:
\[\mathbb{L}^n:=\{x\in\R^{n,1}:\,x_0>0, \,[x,x]=-1\}.\] 
For any $x,y\in \mathbb{L}^n$, the geodesic distance between them is 
\begin{equation}\label{eq:lorentz distance}
    d_{\mathbb{L}^n}(x,y)=\mathrm{arccosh(-[x,y])}
\end{equation}
\paragraph{Notation for Norms and Gradients}
Let $\norm{\cdot}$ and $\langle \cdot, \cdot \rangle$ denote Euclidean $l_2$ norm and inner product. We use $\norm{\cdot}_{\mathbb{D}^n}$ and $\norm{\cdot}_{\mathbb{L}^n}$ to denote norms of Poincar\'{e} and Lorentz vectors. More precisely, for any $v\in T_x\D^n$ and $w\in T_y\LL^n$,
\[\norm{v}_{\D^n}=\lambda_x\norm{v}\quad\text{and}\quad \norm{w}_{\LL^n}=\sqrt{[w,w]}.\]
Similarly, we use $\nabla$ to denote Euclidean gradient and $\nabla_{\mathbb{D}^n}, \nabla_{\mathbb{L}^n}$ for Poincar\'{e} and Lorentz gradient. 

\paragraph{Transition Between the Two Models}
The isometry between the two models can be written explicitly as
$\varphi:\mathbb{D}^n\rightarrow\mathbb{L}^n$ which sends  $x=(x_1,\ldots,x_n)$ to $y:=\big(\frac{\|x\|^2+1}{1-\|x\|^2},\frac{2x_1}{1-\|x\|^2},\ldots,\frac{2x_n}{1-\|x\|^2}\big).$
Its inverse is $\psi:\mathbb{L}^n\rightarrow\mathbb{D}^n$ which sends $y=(y_0,\ldots,y_n)$ to $x:=(\frac{y_1}{1+y_0},\ldots,\frac{y_n}{1+y_0})$. 
The origin $0$ in $\D^n$ corresponds to the point $\bar{0}=(1,0,\ldots,0)$ in $\LL^n$ under these isometries and we call $\bar{0}$ the origin of the Lorentz model.

\paragraph{Operations on $\h^n$}
In Riemannian geometry, there are several important operations relating tangent vectors and points on manifolds,
such as the exponential map, the logarithm map, and parallel transport, which can be written explicitly when considering either model of $\h^n$. 

For example, in the Poincar\'e ball, the exponential map for any $x\in\D^n$ and any $v\in T_x\D^n$ is given by
\begin{equation}\label{eq:poincare exponential}
    \exp_x(v)=x\oplus \tanh\left(\frac{\lambda_x}{2}\norm{v}\right)\frac{v}{\norm{v}}, 
\end{equation}
where $\oplus$ denotes the so-called Möbius addition \citep{ungar2001hyperbolic}. In the Lorentz model, the exponential map for any $x\in\LL^n$ and any $v\in T_x\LL^n$ is given by
\begin{equation}\label{eq:lorentz exponential}
    \exp_x(v)=\cosh(\norm{v}_{\LL^n}) x + \sinh(\norm{v}_{\LL^n}) \frac{v}{\norm{v}_{\LL^n}}.
\end{equation}

\section{Comparing Lorentz and Poincar\'{e} Models}\label{sec:comparing l and p}

As mentioned in the introduction, the volume of a ball in $\h^n$ grows exponentially w.r.t. the radius. This results in that finite trees can be embedded into $\h^n$ with arbitrarily small distortion \citep{sarkar2011low}. However, in order to reach desired small distortion, one needs to scale the embedded tree by a large factor by pushing embedded points towards infinity in $\h^n$ as much as possible, and this comes at a price of numerical representation.

\paragraph{Poincar\'{e} Ball}
To embed a combinatorial tree $T$ with diameter $l$ and maximal degree $D$ up to distortion $(1+\eps)$, one needs $\Omega(\frac{l\log(D)}{\eps})$ bits to correctly represent embedded points in the Poincar\'{e} ball to avoid these points being rounded off to the boundary \citep{sala2018representation}.

We first interpret the number of bits in representing Poincar\'{e} points through a more geometrical lens.

\begin{proposition}[Poincar\'{e} radius]\label{prop:p radius}
For any point $x\in\D^n$, if $\norm{x}=1-10^{-k}$ for some positive number $k$, then in fact, 
\[d_{\D^n}(0,x)=\ln(10)k + \ln(2) + O(10^{-k}).\]
\end{proposition}
See \Cref{proof:p radius} for proof.
Here $k$ is the number of bits required to avoid $x$ being rounded to the boundary. Under the Float64 arithmetic system, the rounding error for the subtraction $1-10^{-k}$ is $2^{-53}$, i.e., $1-10^{-k}$ will be rounded to $1$ when $10^{-k}\leq 2^{-53}\approx 10^{-16}$. Hence, the maximum $k$ is around $16$. This corresponds to a distance $d_{\D^n}(0,x)\approx 38$. 
Thus, we can only represent points correctly within a ball of radius $r_0\approx 38$ in the Poincar\'{e} ball.

\paragraph{Lorentz Model}
Interestingly, it turns out that the Lorentz model has an even smaller representation capacity: it can only represent points correctly within a ball of radius $r_0/2\approx 19$.
Recall the fundamental equation $[x,x]=-1$ in defining the Lorentz model.
Given a point $x=(x_0,\ldots,x_n)\in\LL^n$, if $x_0=10^8$, then $x_0^2=10^{16}$.
Then, the floating number representation of $x_0^2-1$ is the same as the one for $x_0^2$ (roughly speaking, in the Float64 system, the sum/difference of two numbers differ in orders of magnitude over 16 would be rounded to the larger number itself), causing the computation of $[x,x]$ to return $0$ instead of the desired $-1$.  Moreover, due to the following result, we know that the condition $x_0=10^8$ corresponds to the radius $r_0/2$.

\begin{proposition}[Lorentz radius]\label{prop:l radius}
For any point $x\in\LL^n$, if $x_0=10^{k}$ for some positive number $k$, then, 
\[d_{\LL^n}(\bar{0},x)=\ln(10)k + \ln(2) + O(10^{-2k}).\] 
\end{proposition}
See \Cref{proof:l radius} for proof.
This difference in the radii between the two models can be clearly seen from the perspective of the exponential maps. We postpone this interpretation to later \Cref{rmk:2}.

The limitation of the Lorentz model is different from the Poincar\'{e} ball. In the Poincar\'{e} ball, the radius $r_0\approx 38$ is a ``hard'' constraint: any point $r_0$ distant from the origin will collapse to the boundary and there is no way to consider any operations on the hyperbolic space such as computing the distance thereafter. However, in the Lorentz model, the radius $r_0/2\approx 19$ is a ``soft'' constraint: when a point $x=(x_0,x_1,\ldots,x_n)$ is further than $r_0/2$ away from the origin, although there is no viable method to check whether the point is on the Lorentz model, we still have the coordinates to work with: if we let $x_r:=(x_1,\ldots,x_n)$, although $x_0=\sqrt{\norm{x_r}^2+1}$ will be rounded off to $\norm{x_r}$, there is no numerical constraint for the vector $x_r$. In this way, one can still cope with $x_r$ and perform some operations on the Lorentz model. Together with the optimization superiority of the Lorentz model we mention in the next section, this may account for some empirical observations that the Lorentz model is more stable than the Poincar\'{e} ball.

We comment that many works either explicitly or implicitly impose different thresholds in their implementations to restrict all points within a certain radius (\citep{skopek2019mixed},
\citep{nickel2017poincare},
and the popular manifold research toolbox package by \citep{kochurov2020geoopt}, 
with limited discussion on the impact of the choice of these thresholds on the representation capacity and other performance. \textit{The discussion in this section fills the gap and provides a guide for the choice of thresholds.} 

Finally, while the concepts of representation capacity for the Poincaré ball and the Lorentz model initially appear to be distinct, they can be unified through the examination of exponential maps. Although the representation capacity is traditionally defined for points on a manifold within these two models, it can be equivalently specified in the tangent space. This is facilitated by the exponential map $\exp_x:T_x\mathbb{H}^n\to\mathbb{H}^n$, which preserves distance in the radial direction, i.e., $\|v\|=d_{\mathbb{H}^n}(x,\exp_x(v))$ for any $v\in T_x\mathbb{H}^n$.

We then further introduce the concept of \emph{numerical representation capacity} for a hyperbolic model. This is formally defined as the radius of the largest ball, centered at the origin in the tangent space, with the property that all points within this ball can be accurately represented within the hyperbolic model via the exponential map. 
Now, under the exponential map formula \Cref{eq:poincare exponential} for the Poincaré ball, if we take our reference point $x$ to be the origin, then vectors with length larger than 38 will be mapped to the boundary of $\mathbb{D}^n$ which is outside of the manifold. Similarly, in the case of the Lorentz model, under the exponential map formula \Cref{eq:lorentz exponential}, if we choose $x=\bar{0}=(1,0,\ldots,0)$, then vectors with length larger than 19 will be mapped outside of the Lorentz model. Specifically, these vectors will be mapped to the cone $x_0^2=\sum_{i=1}^nx_i^2$. See also later \Cref{rmk:2} for more details. This way of viewing things allows us to see how the concept of numerical representation capacity serves to unite our previous discussions on representation capacity for both the Poincaré and Lorentz models.

\subsection{Optimization}
Although we see that the Poincar\'{e} ball has a larger capacity in representing points than the Lorentz model, we show in this section that this advantage will be wiped out when considering optimization processes on hyperbolic space.

Generally speaking, for any manifold $\mathcal{M}$ and a smooth function $f:\mathcal{M}\rightarrow\R$, in order to solve a Riemannian optimization problem:
\[x^*=\mathrm{argmin}_{x\in\mathcal{M}}f(x),\]
one can apply Riemannian (stochastic) gradient descent \citep{bonnabel2013stochastic} with learning rate $\eta >0$: 
\begin{equation}\label{eq:RSGD}
    x_{t+1}=\exp_{x_t}(-\eta\nabla f(x_t)).
\end{equation}
Theoretically speaking, due to the isometry between $\D^n$ and $\LL^n$, solving optimization problems in $\D^n$ and $\LL^n$ via Riemannian gradient descent should return the ``same'' result. 
To describe this result properly, throughout this section, we will consider a fixed differentiable function $f:\D^n\rightarrow\R$. We consider its corresponding function  $g:\LL^n\rightarrow\R$ on the Lorentz model, i.e., $g=f\circ\psi$ where $\psi:\LL^n\rightarrow\D^n$ is the isometry specified in \Cref{sec:pre}. The functions $f$ and $g$ should be regarded as the same function defined on the hyperbolic space $\h^n$ under two different chart systems. We also illustrate the relationship between $f$ and $g$ through the following commutative diagram.
\[\begin{tikzcd}
  \LL^n \arrow[r,"\psi"] \arrow{dr}[swap]{g}
    & \arrow[l,"\varphi", shift left=0.5ex]\D^n \arrow[d,"f"]\\
&\R \end{tikzcd}\]
In the analysis below, we use $x$ to represent points in $\D^n$ and let $y=\varphi(x)\in\LL^n$ denote its corresponding point in the Lorentz model. Since $\varphi$ and $\psi$ are isometries, the following result holds trivially. 

\begin{proposition}[Gradient descent is the ``same'' for both models]\label{prop:same eta}
For any learning rate $\eta>0$, we have that
\[\psi(\exp_y(-\eta\nabla_{\LL^n} g(y)))=\exp_{x}(-\eta\nabla_{\D^n} f(x)).\]

\end{proposition}
 
Despite this theoretical equivalence, when implementing the Riemannian gradient descent algorithm, the Poincar\'{e} model is more prone to incur the \textit{gradient vanishing problem} than the Lorentz model. 
We first see this through the following computation of the Euclidean norm of Riemannian gradients\footnote{Tangent vectors can be represented using the coordinates provided in either the Poincar\'{e} or the Lorentz model.}. Such norm indicates the magnitude of the coordinates involved in representing the tangent vectors.

\begin{lemma}\label{lm:gradient norm}
For any $x\in \D^n$, assume that $\norm{x}=1-\delta$ for some small positive $\delta$. Then, we have that
\begin{align*}
    \|{\nabla_{\D^n} f(x)}\|&= \Omega\left(\delta^2\norm{\nabla f(x)}\right), \\
    \|{\nabla_{\LL^n} g(y)}\|&= O\left(\|{\nabla f(x)}\|\right).
\end{align*}
\end{lemma}
See proof from \Cref{proof:gradient norm} and note that in a special case where the vectors $x$ and $\nabla f(x)$ are parallel, the big O for $\|{\nabla_{\LL^n} g(y)}\|$ can be replaced by $\Omega$. Notice that $\delta$ is small when $x$ is close to the boundary of the Poincar\'{e} ball. As a consequence, optimizing through the Poincar\'{e} model results in dealing with numbers smaller in order of magnitude than through the Lorentz model. 
This suggests a more pronounced gradient vanishing problem for the Poincar\'{e} model when points are near the boundary and a potential accumulation of rounding error, e.g., when points are on their way to the boundary to achieve a lower distortion for the task of tree embedding. 

To clearly derive the limitation of the Poincar\'e model in optimization, we consider Taylor expansions of one step of the Riemannian gradient descent (\Cref{eq:RSGD}) for Poincar\'{e} ball and the Lorentz model as follows. 

\begin{theorem}\label{thm:gd p vs l}
Let  $\delta=10^{-k}$ be a small positive number. Without loss of generality, consider the point $x=(1-\delta,0,\ldots,0)\in\mathbb{D}^n$. Assume that $\nabla f(x)=(\partial_1f(x),0,\ldots,0)$\footnote{We only consider this direction as it is the most relevant direction in understanding the phenomenon of pushing points towards the boundary of $\D^n$.} and $\partial_1f(x)<0$. Let $y:=\varphi(x)\in\mathbb{L}^n$. If we let $E:=\|{\nabla f(x)}\|$, then
\begin{align*}
    \exp_x(-\eta\nabla_{\D^n} f(x))&\!=\! (1-10^{-k} +O(\eta E10^{-2k}),0,\ldots,0), \\
    \exp_y(-\eta\nabla_{\LL^n} g(y))&\!=\!\\
     10^k\bigg(1-\frac{1}{2}\cdot 10^{-k}&+\eta E10^{-k}+O(\eta E10^{-2k}), \\
    1-\frac{1}{2}\cdot 10^{-k}&+\eta E10^{-k}+O(\eta E10^{-2k}),0,\ldots,0\bigg).
\end{align*}    
\end{theorem}
See \Cref{proof:thm1} for proof.
To interpret the results, assume for simplicity that $\eta=1$ and $E=\norm{\nabla f}=O(1)$ is bounded. 
Then, at each step of Riemannian gradient descent, the Poincar\'e ball needs to compute the sum of two numbers (i.e., $1-10^{-k}$ and $O(\eta E10^{-2k})=O(10^{-2k})$) that differ by $2k$ in orders of magnitude whereas in the Lorentz case only differ by $k$ (i.e., $1-\frac{1}{2}\cdot10^{-k}$ and $\eta E10^{-k}+O(\eta E10^{-2k})=O(10^{-k})$). 
We emphasize that the terms $O(10^{-2k})$ and $O(10^{-k})$ do not represent error terms, but instead indicate the update in one step of Riemannian gradient descent. As a result, the Poincaré model has a smaller update term, which causes it to suffer from a more severe gradient vanishing issue.
In particular, when $k$ is chosen to be $8$, the gradient term $O(10^{-2k})=O(10^{-16})$ which is around the same order as the rounding error. Hence this term will be neglected and the Poincar\'{e} gradient descent will be stuck. However, for the Lorentz model, if one ignores the numerical representation issue, then from the optimization perspective, Lorentz model allows $\delta$ to be as small as $10^{-16}$ without encountering the severe gradient vanishing problem the Poincar\'e ball suffers from.
Of course one could use a large learning rate $\eta=10^8$ so that the Poincar\'e gradient numerical representation is similar to the one for the Lorentz. But such a large learning rate will induce instability in the training process.
Note by \Cref{prop:p radius}, $\delta=10^{-8}$ corresponds to a point $19$ away from the origin. Hence, even though the Poincar\'e ball can represent points up to 38 away from the origin, a large part of the region cannot be utilized in optimization.

\section{Euclidean Parametrization of Hyperbolic Space}\label{sec:euclidean trick}
In the previous section, we see that the Poincar\'e model and the Lorentz model have different pros/cons from two different perspectives: point-representation and optimization. It is thus natural to ask whether there is a way to represent hyperbolic space that could prevail in both perspectives: allowing to represent points within a relatively big region like the Poincar\'{e} ball (i.e., of radius 38) or even larger, while in terms of optimization behaves like the Lorentz model.

In this section, we argue in \Cref{section:feature parametrization} that using a Euclidean parametrization for the hyperbolic space can help address these limitations. 
Such Euclidean parameterization for hyperbolic space has previously been utilized in \citep{mathieu2019continuous} for training hyperbolic features involved in hyperbolic VAE. However, here we present a new perspective that can provide a numerically more robust proxy for optimization on hyperbolic space.
In \Cref{sec:hyperplane}, {we consider a polar coordinate form of this parameterization and extend it to parametrize hyperbolic hyperplanes through a more succinct derivation than the one given in \citep{shimizu2020hyperbolic} and we also correct a mistake in their derivation (see \Cref{rmk:2} and \Cref{rmk:shimizu}).} Finally, we observe that the hyperplane parametrization resolves the nonconvexity condition for the Lorentz hyperplane. In this way, we further apply this parametrization to and show a performance boost of the Lorentz SVM \citep{cho2019large}.

\subsection{Feature Parametrization}
\label{section:feature parametrization}

One property inherent to the negative curvature of the hyperbolic space is that the exponential map at any point is a diffeomorphism (see the Cartan–Hadamard theorem in Riemannian geometry). 
This gives rise to a natural Euclidean parametrization of features in the hyperbolic space via the tangent space and the exponential map: pick a point $p\in \h^n$ and then consider the exponential map $\exp_p:T_p\h^n\rightarrow\h^n$; as $T_p\h^n$ can be identified (see the section below) as $\R^n$, we then have a Euclidean parametrization of $\h^n$.
 
In either the Poincar\'{e} ball or the Lorentz model, there exists a canonical choice of point $p$: the origin $0$ in Poincar\'{e} ball and the origin $\bar{0}:=(1,0,\ldots,0)$ in the Lorentz model. 
Consider the map $\R^n\rightarrow T_0\D^n$ sending $z$ to $\frac{z}{2}$ and the map $\R^n\rightarrow T_{\bar{0}}\LL^n$ sending $z$ to $(0,z)$.
These two maps are both isometries, i.e., $\norm{z}$ agrees with the Riemannian norms of the images of $z$ in $T_0\D^n$ and $T_{\bar{0}}\LL^n$.

Now, we specify how hyperbolic features are recovered through respective exponential maps: let $z \in \mathbb{R}^n$, then $z$ parametrizes Poincar\'{e} features 
\begin{equation}\label{eq:divide 2}
   x= \exp_0\left(\frac{z}{2}\right) = \tanh\left(\frac{\norm{z}}{2}\right) \frac{z}{\norm{z}} 
\end{equation}
and Lorentz features
\begin{equation}\label{eq:lorentz para}
     y = \exp_{\bar{0}}((0, z))= \left(\cosh(\norm{z}), \sinh(\norm{z}) \frac{z}{\norm{z}}\right).
\end{equation}
\looseness=-1We denote these maps by $F_{D}:\R^n\rightarrow\D^n$ and $F_L:\R^n\rightarrow\LL^n$, respectively. They are, of course, sending a point $z\in \R^n$ to the ``same'' point in the two models under the isometry $\varphi:\D^n\rightarrow\LL^n$, i.e., the following diagram commutes.
\[\begin{tikzcd}
  \R^n \arrow[r,"F_D"] \arrow[dr, "F_L"]
    & \D^n \arrow[d,"\varphi"]\\
&\LL^n \end{tikzcd}\]
Furthermore, these two maps preserve distances between a point and the origin, i.e.,
\[d_{\D^n}(0,x)=d_{\LL^n}(\bar{0},y)=\norm{z}.\]
\begin{remark}[The mysterious ``2"]
\label{rmk:2}Note that the division by 2 in the Poincar\'{e} case (see \Cref{eq:divide 2}) is crucial to make the above equation hold. 
The missing of this division by 2 is prevalent in the literature \citep{shimizu2020hyperbolic} and may raise an unfair comparison between Poincar\'{e} and Lorentz models (see more details in later \Cref{rmk:shimizu}). 
We also remark that under the 64 bit arithmetic system, $\tanh(t)$ will be rounded to 1 as long as $t>19$. Hence, one can compute $\exp_0$ correctly up to radius $r_0\approx 2*19$. As for the Lorentz model, since the float representations of $\cosh(t)$ and $\sinh(t)$ will be the same as $t>19$ (note that $\tanh(t)=\frac{\sinh(t)}{\cosh(t)}$), the Lorentz $\exp_{\bar{0}}$ can only represent points correctly within a ball with radius up to $r_0/2\approx 19$.
\end{remark}

This Euclidean parametrization of course allows \textit{representing any point in the hyperbolic space without concern regarding numerical limitations} found in the case of the Poincar\'e ball or the Lorentz model. However, it is important to note that the Euclidean parametrization does not preserve the Riemannian structure of the hyperbolic space. Nonetheless, we consider this parametrization a viable way to train functions defined on hyperbolic spaces. Next we will provide a more precise analysis of the behavior of this parametrization in optimization.

\paragraph{Optimization}
Given this Euclidean parametrization, one can translate hyperbolic optimization problems into a Euclidean one:
Given any differentiable function $f:\D^n\rightarrow \R$, we define $h:\R^n\rightarrow\R$ by letting $h:=f\circ F_D$. Of course, if we consider the function $g:=f\circ\psi:\LL^n\rightarrow\R$ defined on the Lorentz model, the composition $g\circ F_L$ agrees with $h$; see the commutative diagram below.
\[\begin{tikzcd}
	{\R^n} & {\mathbb{D}^n} \\
	& {\mathbb{L}^n} & \R
	\arrow["{F_D}", from=1-1, to=1-2]
	\arrow["{F_L}", from=1-1, to=2-2]
	\arrow["g", from=2-2, to=2-3]
	\arrow["f", from=1-2, to=2-3]
	\arrow["h"', from=1-1, to=2-3, bend right=45pt]
    \arrow["\psi"', from=2-2, to=1-2]
\end{tikzcd}\]
Hence, in the discussion below, for simplicity, we start with a function defined on the Poincar\'{e} ball.

Consider the optimization problem on the hyperbolic space $\D^n$ given by:
\begin{equation}\label{eq:r opt}
    x^*=\mathrm{argmin}_{x\in \D^n}\,{f}(x).
\end{equation}
This problem can be transformed into an optimization problem on the Euclidean space $\R^n$ through the Euclidean parametrization $F_D$, resulting in the following optimization problem:
\begin{equation}\label{eq:alt r opt}
    z^*=\mathrm{argmin}_{z\in \R^n}\,h(z)=\mathrm{argmin}_{z\in \R^n}\,f(F_D(z)).
\end{equation}
Since $F_D$ is a bijective diffeomorphism, any optimizer $z^*$ of \Cref{eq:alt r opt} corresponds to an optimizer $F_D(z^*)$ of \Cref{eq:r opt}, and vice versa.

\textit{The newly obtained optimization problem will suffer less from the gradient vanishing problem than the Poincar\'{e} model and has a similar performance to the Lorentz gradient descent. }
We analyze the one step gradient descent w.r.t. $h$ in a manner similar to \Cref{thm:gd p vs l} below.

\begin{theorem}\label{thm:euclidean}
 
Under the same assumptions and notation as in \Cref{thm:gd p vs l}, if we let $z:= F_D^{-1}(x)=(2\mathrm{arctanh}(1-\delta),0,\ldots,0)\in\mathbb{R}^n$, then $\norm{\nabla h(z)}=\Omega(\delta\norm{\nabla f(x)})$ and
\begin{align*}
    &z-\eta\nabla h(z)\\
    &=\ln(10)k+\ln(2)+({\eta \norm{\nabla f(x)}}-1/2)10^{-k}+O(10^{-2k}).
\end{align*}
\end{theorem}
See \Cref{proof:thm2} for proof. Similarly, as in the discussion below \Cref{thm:gd p vs l}, we assume $\eta=1$ and $\norm{\nabla f(x)}=O(1)$. In this case, we see the gradient update step involves the sum of two terms whose orders of magnitude differ roughly by $\log_{10}(k)+k$. 
This difference in orders of magnitude sits between the Poincar\'e case and the Lorentz case. 
In particular, when $k$ is smaller than $8$ (recall in \Cref{sec:comparing l and p} the Lorentz model can only correctly represent points when $k\leq 8$), then the difference in orders of magnitude for the Euclidean gradient descent is almost the same as the one for the Lorentz gradient descent. 
Hence, the Euclidean parametrization has similar performance to the Lorentz model in optimization when points are within a reasonable range. We empirically validate this point in tree embeddings experiments in \Cref{sec:epl}.

\begin{remark}
\cite{guo2022clipped} claimed that maps of the form $\R^n\xrightarrow{\exp_0}\D^n\xrightarrow{f}\R$ will incur gradient vanishing problem due to the fact that the Riemannian gradient of $f$ will be small when the point is near the boundary of the Poincar\'e ball (see \Cref{lm:gradient norm}). 
    However, this is only half of the story. In fact, the Jacobian of the map $\exp_0$ will reduce the power of $\delta$ from 2 to 1 and thus help alleviate the gradient vanishing problem (see \Cref{lm:gradient norm}, \Cref{thm:euclidean} and \Cref{proof:thm2} for more details).
\end{remark}
 
\subsection{Hyperplane Parametrization}\label{sec:hyperplane}
The notion of hyperplanes in the hyperbolic space is a natural analogue of the notion of subspaces in the Euclidean spaces. This analogy results in that hyperbolic hyperplanes have been used for hyperbolic multinomial logistic regression (MLR) in \citep{ganea2018hyperbolic}, in designing hyperbolic neural networks in \citep{shimizu2020hyperbolic}, and of course, in designing hyperbolic SVM \citep{cho2019large}.

Consider $p \in \mathbb{L}^n, w \in T_p\mathbb{L}^{n}$, the Lorentz hyperplane passing through $p$ and perpendicular to $w$ is denoted by 
\begin{equation}\label{eq:lorentz hyperplane}
    H_{w, p} = \{x \in \LL^{n} | \; [w, x] = 0 \}.
\end{equation}
Notice that $p$ is implicit in the condition since $w$ is on the tangent space of $p$. In hyperbolic MLR or hyperbolic neural networks, one needs to optimize over the choice of hyperplanes. \Cref{eq:lorentz hyperplane} gives us the opportunity to only use $w$ for optimization. 
However, in order for $w$ to be a feasible tangent vector, we need to impose the nonconvex restriction that $[w, w] > 0$ which is undesirable.

Instead, we derive a Euclidean parametrization of $H_{w,p}$ to get rid of the nonconvex constraint. This parametrization follows closely the Euclidean parametrization of hyperbolic features.
We first parametrize $p$ as in the previous section via $\exp_{\bar{0}}$ but with a flavour of polar coordinates. 
We then parametrize $w$ as the parallel transport of a vector from $T_{\bar{0}}\LL^n$ to $T_p\LL^n$. Now, more precisely, 
let $a \in \mathbb{R}, z \in \mathbb{R}^n$, and set $\bar{z} := (0, z) \in {T}_{\bar{0}}\mathbb{L}^n$. We set 
\begin{equation}\label{eq:e para l hyp}
    p := \exp_{\bar{0}}\left(a \frac{\bar{z}}{\norm{\bar{z}}}\right)=\left(\cosh(a),\sinh(a) \frac{z}{\norm{z}}\right).
\end{equation}
Note that this equation is just the polar coordinate version of \Cref{eq:lorentz para}. Then, we let
\begin{equation}\label{eq:vector parametrization}
    w := \text{PT}_{\bar{0} \mapsto p}(\bar{z})=\left(\text{sinh}(a)\norm{z}, \;  \text{cosh}(a) z\right)
\end{equation}
Here $\text{PT}_{x \mapsto y}$ denotes the parallel transport map along the unique geodesic from $x$ to $y$ and see a succinct derivation of the above formula in \Cref{proof:derivation}.

Hence, the parametrized hyperplane becomes 
$$
\tilde{H}_{z, a} = \{x\in \LL^{n} | \; \text{cosh}(a) \langle z,x_r\rangle = \text{sinh}(a)\norm{z}x_0 \},
$$
where $x_r=(x_1,\ldots,x_n)$ denotes the latter $n$ components of $x=(x_0,x_1,\ldots,x_n)$.
The geometric meaning of this parametrization is as follows: the hyperplane is passing through a point $p$ which is $|a|$ far away from the origin and along the direction $z$, and is perpendicular to a vector $w$ which has length $\norm{z}$.

\begin{remark}\label{rmk:shimizu}
\looseness=-1    \cite{shimizu2020hyperbolic}  parametrized hyperplanes in the Poincar\'{e} ball for the purpose of reducing the number of training parameters from $2n$ to $n + 1$ for hyperplanes in $\D^n$. Our derivation starts from the Lorentz model and is more succinct from the approach via the Poincar\'{e} ball. Eventually, we obtain essentially the formulas up to a small difference: The terms $\cosh(a)$ and $\sinh(a)$ in the formula above are replaced by $\cosh(2a)$ and $\sinh(2a)$ in \citep{shimizu2020hyperbolic}. This difference is due to the mistake that they did not take into account the conformal factor $\lambda_0^2=4$ of $T_0\D^n$. See also \Cref{rmk:2}. 
\end{remark}

\begin{remark}[What if $z=0$?]
    The careful reader may notice that \Cref{eq:e para l hyp} does not work for the case when $z=0$. We note that however, the final formula \Cref{eq:vector parametrization} works for the case when $z=0$. When $z=0$, the tangent vector $w$ becomes the zero vector at $p$ and thus the notion of hyperplane degraded to the whole space.
\end{remark}

Now, we see there is no restriction on $z$ and $a$ and thus we get rid of the nonconvex condition $[w,w]>0$ for parametrizing hyperplanes. This substantially simplifies the optimization process. Next we demonstrate how this parametrization can be used for hyperbolic SVM.

\subsubsection{A New Formulation of Hyperbolic SVM}

Support Vector Machine (SVM) is a classic machine learning model for both classification and regression by training a separating hyperplane. 
\citep{cho2019large} first generalized SVM to datasets in the hyperbolic space for learning a separating hyperbolic hyperplane through the Lorentz model. Consider a training set $\{((x_i), (y_i)) \}_{i=1}^n$ where $x_i \in \mathbb{L}^n, y_i \in \{1, -1\}, \forall i \in \{1, ..., n\}$. Then, the (soft-margin) Lorentz SVM (LSVM) can be formulated as
\begin{align*}
    \underset{w \in \mathbb{R}^{n + 1}, [w,w]>0}{\min} \; &\frac{1}{2} \norm{w}_{\mathbb{L}^n}^2 + C\sum_{i = 1}^n l_{\mathbb{L}^n}(-y_i [w, x_i]) 
\end{align*}
where $l_{\mathbb{L}^n}(z) = \max(0, \text{arcsinh}(1) - \text{arcsinh}(z))$ is a variant of the hinge function that respects the Lorentz distance and $C$ is a scalar controlling the tolerance of misclassification. 
This problem has both nonconvex constraint and nonconvex objective function and the authors propose using a projected gradient descent optimization.

As the vector $w$ in the above formulation serves as the tangent vector determining a Lorentz hyperplane (cf. \Cref{eq:lorentz hyperplane}), we can then parametrize $w$ via \Cref{eq:vector parametrization}. As $w$ is the parallel transport of $\bar{z}$, we have that $\norm{w}_{\LL^n}=\norm{\bar{z}}_{\LL^n}=\norm{z}.$
In this way, we obtain the following minimization problem over parameters $z\in\R^n$ and $a\in\R$ without any nonconvex constraints. 
\begin{align*}
    &\underset{z \in \mathbb{R}^{n}, a \in \mathbb{R}}{\min} \; \frac{1}{2} \norm{z}^2 +
    \\
   &C\sum_{i = 1}^n l_{\mathbb{L}^n}\!\left(y_i (\sinh(a)\norm{z}x_0  - \cosh(a) \langle z,x_r\rangle)\right)
\end{align*}
Empirical evaluation in \Cref{sec:hyperbolic svm experiment} shows that this new SVM framework (LSVMPP) outperforms LSVM, possibly by escaping the local minima.

\section{Experiments}
\label{sec: experiments}

\subsection{Hyperbolic Embeddings}
\label{sec:epl}

This subsection aims to empirically validate our theoretical results (\Cref{thm:gd p vs l} and \Cref{thm:euclidean}) by demonstrating that both the Lorentz and the Euclidean (parametrized) models better leverage the hyperbolic structure than the Poincar\'{e} ball. To do so we embed trees in hyperbolic spaces.
The embedding performance is measured by distortion, max distortion, and diameter. Distortion quantifies the fidelity of embedding w.r.t. the tree structure by simultaneously considering the expansion and contraction of the embedding. In a general setting, let $f: (M_R, d_R) \rightarrow (M_E, d_E)$ be an embedding map between two metric spaces, the distortion of the map is given by $\delta = \delta_{\text{contraction}} \cdot \delta_{\text{expansion}}$ where
\[ \delta_{\text{contraction}} = \underset{x\neq y}{\text{mean}} \frac{d_R(x, y)}{d_E(x, y)} \;; \;
    \delta_{\text{expansion}} = \underset{x\neq y}{\text{mean}} \frac{d_E(x, y)}{d_R(x, y)}    \]
Max distortion $\delta^{\max}$ is defined similarly by replacing the mean with the supremum.

\textbf{Datasets}
We simulated tree datasets $S$ in $\mathbb{R}^2$ with varying degrees of complexity whose distances are equally weighted graph distances. \Cref{fig:sim_tree_vis} lists two example trees with more in \Cref{tab:sim_tree_vis_full}.
A detailed discussion on the impact of performance by its shape can be found in the \Cref{subsec:syn_tree_dgp}.

\begin{figure}[tbp]
    \centering
    \begin{subfigure}
        \centering
        \includegraphics[width=0.23\textwidth]{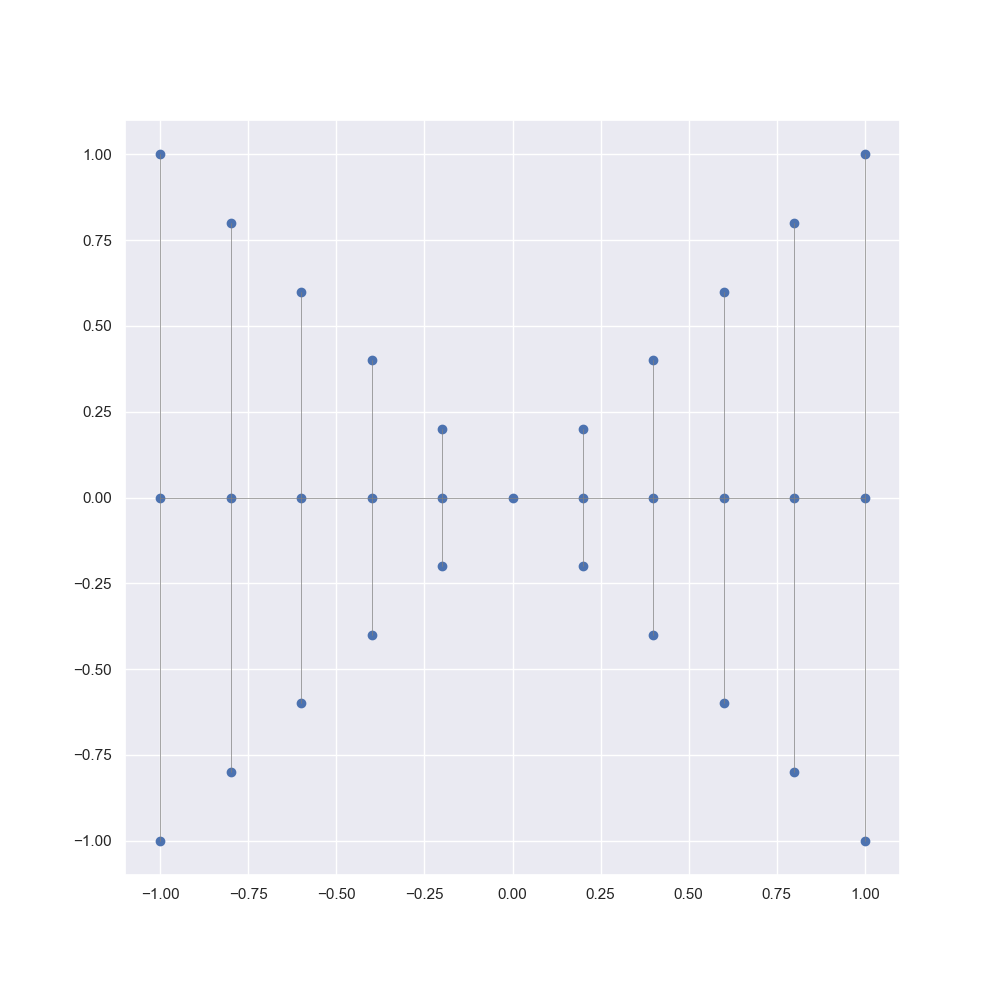}
    \end{subfigure} 
    \hfill
    \begin{subfigure}
        \centering
        \includegraphics[width=0.23\textwidth]{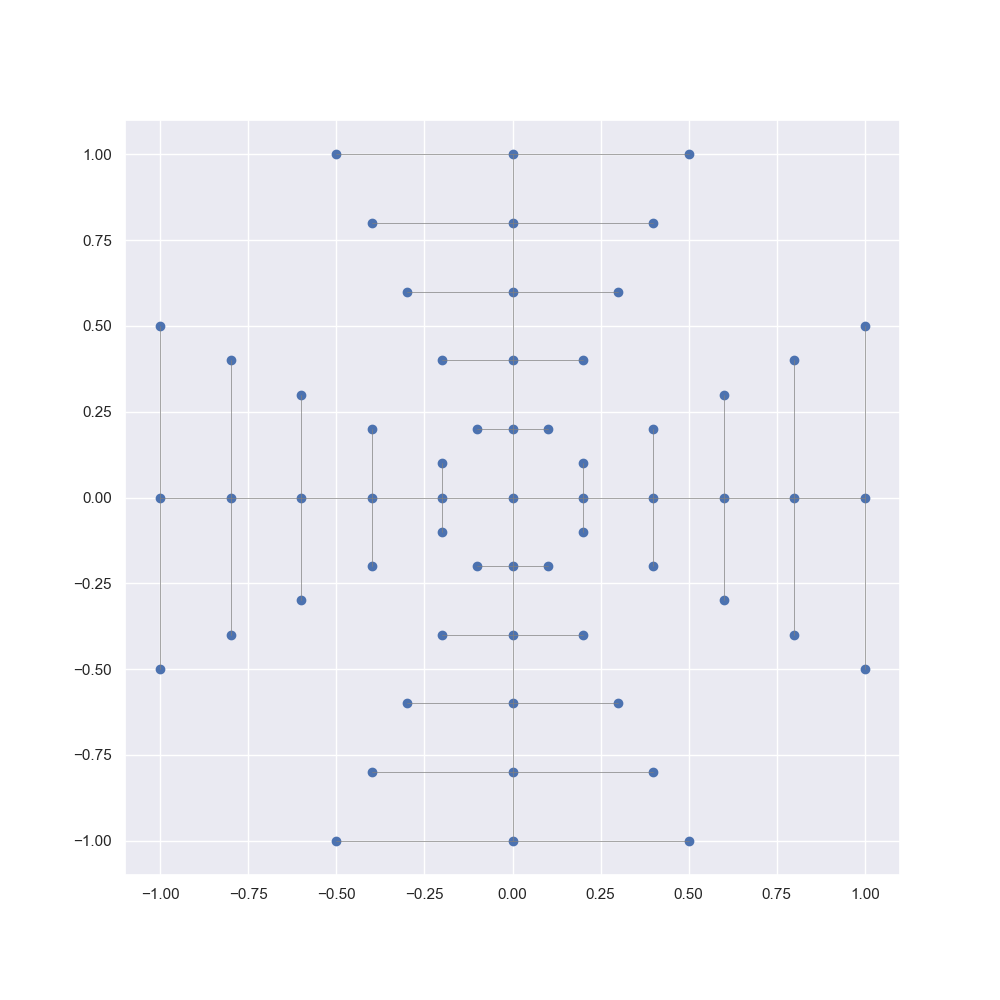}
    \end{subfigure}
    \caption{Simulated Tree 1 and Tree 2 in $\mathbb{R}^2$. The 2D coordinates of each node are features and pairwise distances are computed through shortest path distance on the connected graph.} 
    \label{fig:sim_tree_vis}
\end{figure}

\textbf{Experiments}
 We optimize the following loss function to minimize embedding distortion:
\begin{equation*}
    \mathcal{L}(\theta) = \frac{1}{|S|(|S|-1)}\sum_{x \ne y \in S}\left(\frac{d_E(x, y)}{z_E} - \frac{d_R(x, y)}{z_R}\right) ^ 2 
\end{equation*}
where $\theta$ collects the parameters in the network, $d_E$ denotes the embedding distance and $d_R$ denotes the original tree distance, and $z_E, z_R$ are averages of $d_E$ and $d_R$, respectively.
For $\mathbb{L}^2$ and $\mathbb{D}^2$, the features to update are the results of the respective exponential maps $F_L, F_D$ of each dataset (treating Euclidean data points as in the tangent space of the origin). For the Euclidean model $\mathbb{E}^2$, the features to update are the raw $\mathbb{R}^2$ dataset, but the pairwise distances are computed after $F_L$ or equivalently $F_D$.
We use Riemannian SGD \citep{becigneul2018riemannian} for hyperbolic models and SGD for the Euclidean model, fixing a learning rate of 1 and train for 30000 epochs.

\textbf{Results}
We report the metrics after the final epoch in \Cref{tab:average_distortion_full}. Here raw manifold refers to using Euclidean distances as embedding distances against tree distances to measure distortions (i.e. original distortions). 
In both cases, the benefit of Euclidean parametrization to train hyperbolic models is clear: the Euclidean model achieves the lowest average distortions and max distortion, although the performances of Lorentz and Euclidean are comparable as expected. Furthermore, the diameters of the resulting embeddings for the Lorentz and Euclidean models are much larger than the Poincar\'{e} embeddings. This validates that Poincar\'{e} tends to get stuck using Riemannian optimization methods. 

%
\begin{table}[htbp!]
    \centering
    \caption{Average distortion $\delta$, max distortion $\delta^{\max}$, and diameter $d$ by Dataset created in \Cref{fig:sim_tree_vis} and Manifold}
    \begin{tabular}{cc|ccc}\toprule
    tree & manifold & $\delta$ & $\delta^{\max}$ & $d$ \\ 
    \midrule 
    \multirow{4}{*}{1} & raw            & 1.1954 & 10.6062 & 2.8284 \\
                       & $\mathbb{D}^2$ & 1.0462 & 4.0546 & 4.8740 \\
                       & $\mathbb{L}^2$ & 1.0176 & 2.4511 & 10.1827 \\ 
                       & $\mathbb{E}^2$ & \textbf{1.0157} & \textbf{2.3158} & \textbf{10.9019} \\
    \midrule
    \multirow{4}{*}{2} & raw            & 1.1186 & 14.1421 & 2.2361 \\ 
                       & $\mathbb{D}^2$ & 1.0589 & 7.1435 & 4.9757 \\ 
                       & $\mathbb{L}^2$ & 1.0180 & 3.2803 & 10.6719 \\ 
                       & $\mathbb{E}^2$ & \textbf{1.0134} & \textbf{2.7408} & \textbf{11.4504} \\
    \bottomrule
    \end{tabular}
    \label{tab:average_distortion}
\end{table}
%



\Cref{tab:sim_tree_vis_emb_demo} visualizes the embeddings of Tree 1 and Tree 2 in \Cref{fig:sim_tree_vis} (see also \Cref{tab:sim_tree_vis_emb_part1} and \Cref{tab:sim_tree_vis_emb_part2} in \Cref{sec:tree opt} for visualizations of embeddings for all simulated trees). Lorentz features are stereographically projected to the Poincar\'{e} ball for comparisons. Dimmer points are closer to the edge of the input space while darker points are more at the center. The red point is the root or center of the tree. Poincar\'{e} results are seemingly not as spread out as the other two methods, and resembling the early stage in the optimization process of Euclidean and Lorentz models. 

\begin{figure}[H]
    \centering
    \begin{subfigure}
        \centering
        \includegraphics[width=0.15\textwidth]{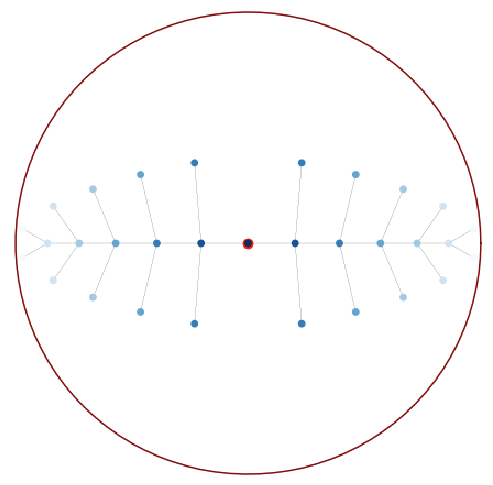}
    \end{subfigure} 
    \hfill 
    \begin{subfigure}
        \centering
        \includegraphics[width=0.15\textwidth]{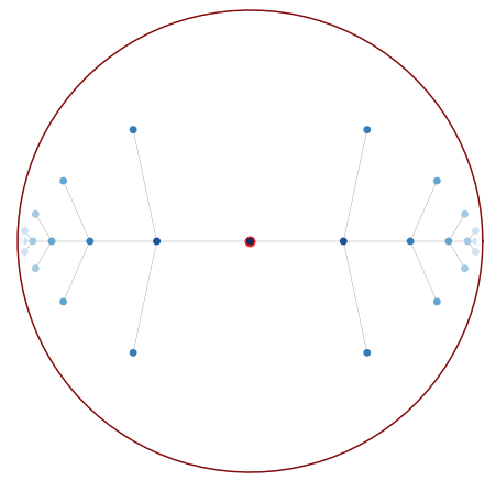}
    \end{subfigure} 
    \hfill 
    \begin{subfigure}
        \centering
        \includegraphics[width=0.15\textwidth]{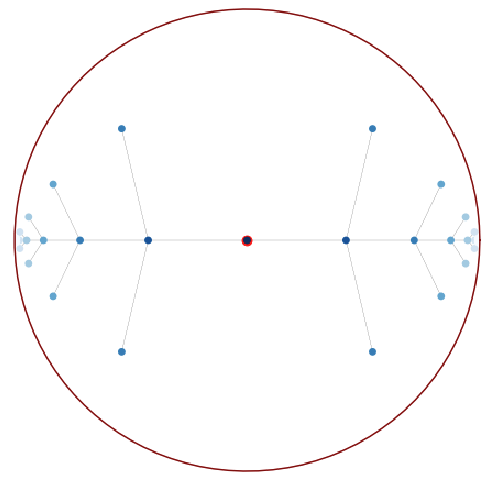}
    \end{subfigure}  \\

    \begin{subfigure}
        \centering
        \includegraphics[width=0.15\textwidth]{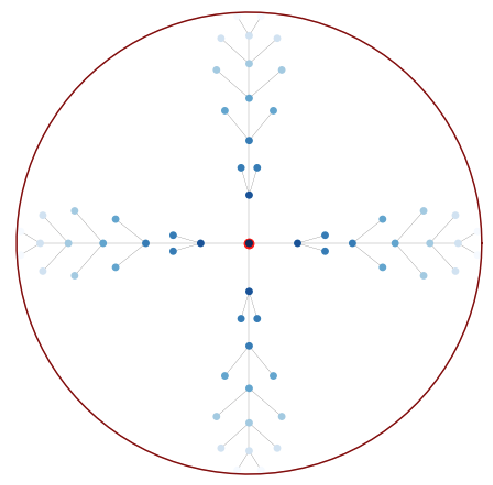}
    \end{subfigure} 
    \hfill 
    \begin{subfigure}
        \centering
        \includegraphics[width=0.15\textwidth]{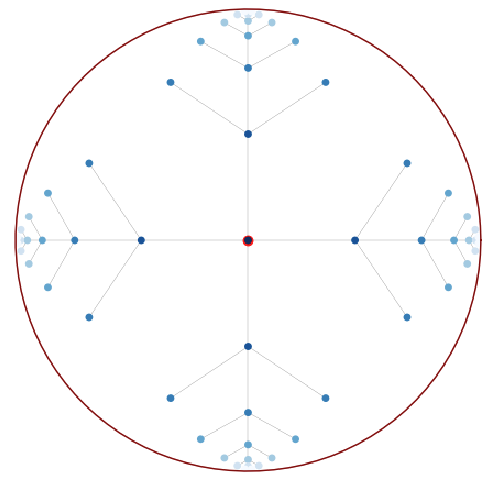}
    \end{subfigure} 
    \hfill
    \begin{subfigure}
        \centering
        \includegraphics[width=0.15\textwidth]{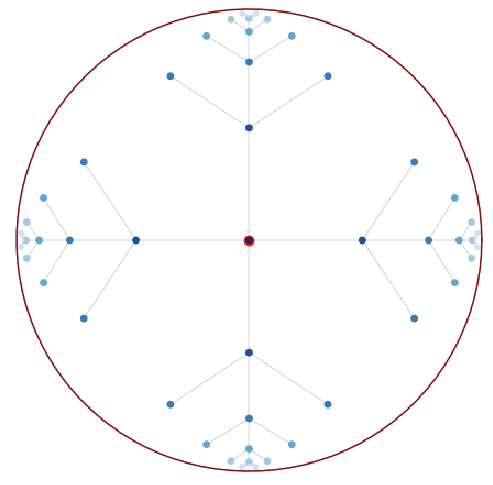}
    \end{subfigure} 
    \caption{Embedding of Tree 1 (top row) and Tree 2 (bottom row) at the final epoch with Poincar\'{e} (left), Lorentz (mid), and Euclidean parametrization (right).}
    \label{tab:sim_tree_vis_emb_demo}
\end{figure}

Finally, in \Cref{fig:rgrad_ratio}, we directly compare the median norms of Riemannian gradient from respective manifolds along different training epochs, whose median is taken across different embedding points. We visualize two ratios, one between gradient norms of $\mathbb{L}^2$ and $\mathbb{E}^2$ (blue) and the other between $\mathbb{D}^2$ and $\mathbb{E}^2$ (red). In both cases, Poincaré graidents have much smaller a norm than its competitors. This further confirms its optimization disadvantage.

\begin{figure}[htbp!]
    \centering
    \begin{subfigure}
        \centering
        \includegraphics[width=0.45\linewidth]{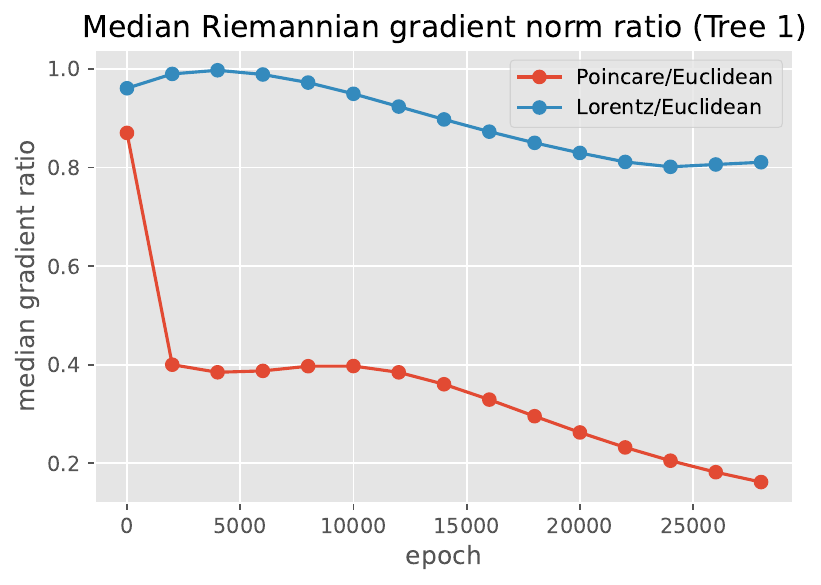}
    \end{subfigure} \hfill
    \begin{subfigure}
        \centering
        \includegraphics[width=0.45\linewidth]{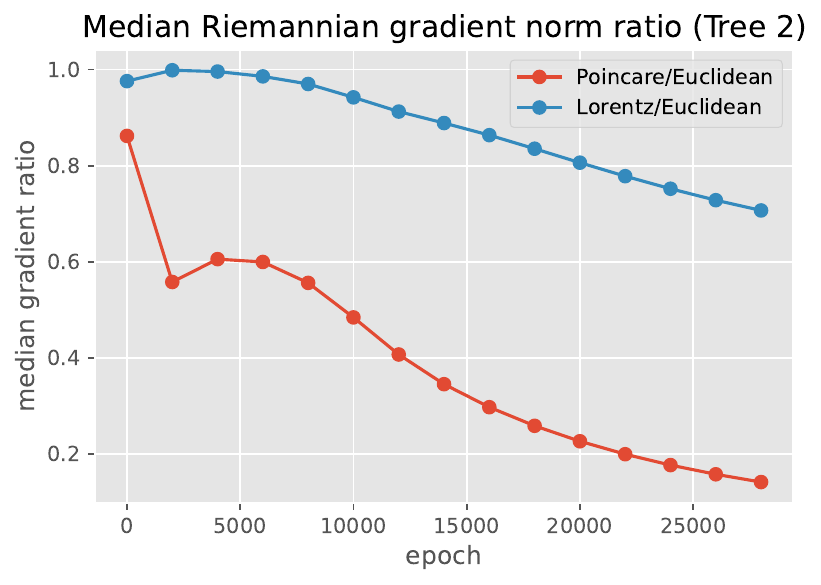}
    \end{subfigure}
    \caption{Median Riemannian Gradient Norm Ratios by epoch for Tree 1 (left) and Tree 2 (right): Poincaré gradients have significantly smaller norms than others}
    \label{fig:rgrad_ratio}
\end{figure}

\subsection{Hyperbolic SVM Models}\label{sec:hyperbolic svm experiment}
We tested the performances of all models in multi-class classification tasks using the following six simulated datasets and six real datasets.

\textbf{Synthetic Datasets}
We simulated two sets of data: \textit{Gaussian mixtures} and \textit{explicit trees}. For Gaussian mixtures, we simulate each dataset using 1200 points on the Poincar\'{e} disk $\mathbb{D}^2$ following the settings in \citep{cho2019large}. Specifically, we first generate a Gaussian mixture on $\mathbb{R}^2$, whose $k$ centroids follow $N(\mathbf{0}, 1.5\mathbb{I}_2)$ equipped with $N(\mathbf{0}, \mathbb{I}_2)$ noise. Each of the $k$ clusters is given a unique label and has the same number of points. We then use the map $F_D$ to map Euclidean features to Poincare features.
We tested performances on $k \in \{3, 5, 10\}$. {For each data set of explicit trees, we initialize a tree in $\mathbb{R}^2$ and then embed the tree into the hyperbolic space using Euclidean parametrization as described in \Cref{sec:epl}. Then we randomly select a node and treat all of its descendants as positive class and negative class otherwise (i.e. $k=2$). This subtree-classification also resembles the WordNet subtree classification presented in \citep{cho2019large}. See \Cref{app:svm_syn_data_gen} for more details.}

\textbf{Real Datasets}
We tested the performances on three datasets: CIFAR-10 \citep{krizhevsky2009learning}, fashion-MNIST \citep{xiao2017fashion}, Paul Myeloid Progenitors developmental dataset \citep{paul2015transcriptional}, Olsson Single-Cell RNA sequencing dataset \citep{olsson2016single}, Krumsiek Simulated Myeloid Progenitors \citep{krumsiek2011hierarchical}, and Moignard blood cell developmental trace from single-cell gene expression \citep{moignard2015decoding}. CIFAR-10 and Fashion-MNIST each contain images of 10 different types of objects. Other datasets are rooted in biological applications: Paul has 19 cell types developed from the myeloid progenitors; Krumsiek likewise but with a different latent structure and 11 classes; Olsson is a smaller dataset with RNA-seq data of 8 types; Moignard has 7 classes for different blood cell developmental stages. Each dataset is embedded into $\mathbb{D}^2$ with the default curvature 1 using the technique introduced in \citep{khrulkov2020hyperbolic,nickel2017poincare}. We visualize three of the tested datasets in \Cref{fig:svm_data}.

\begin{figure}[htbp]
    \centering
    \begin{subfigure}
        \centering
        \includegraphics[width=0.15\textwidth,trim={1.5cm 1cm 0 1cm},clip]{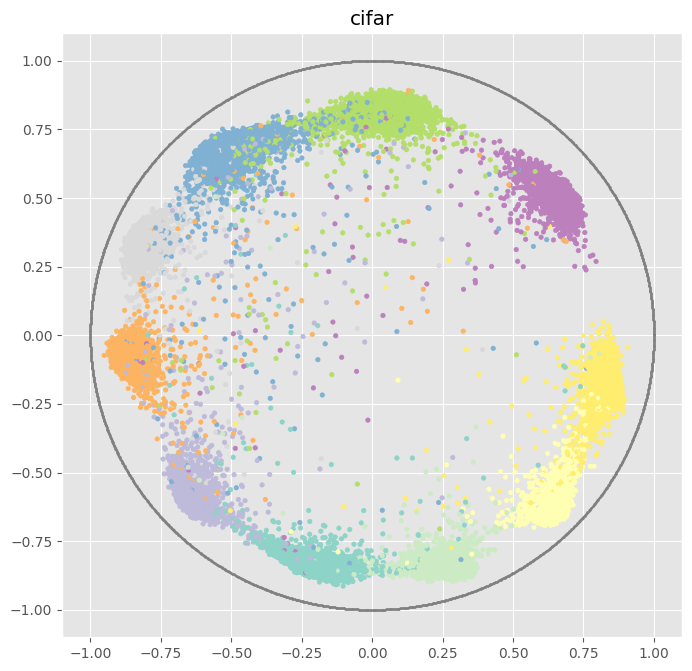}
    \end{subfigure}
    \hfill
    \begin{subfigure}
        \centering
        \includegraphics[width=0.15\textwidth,trim={1.5cm 1cm 0 1cm},clip]{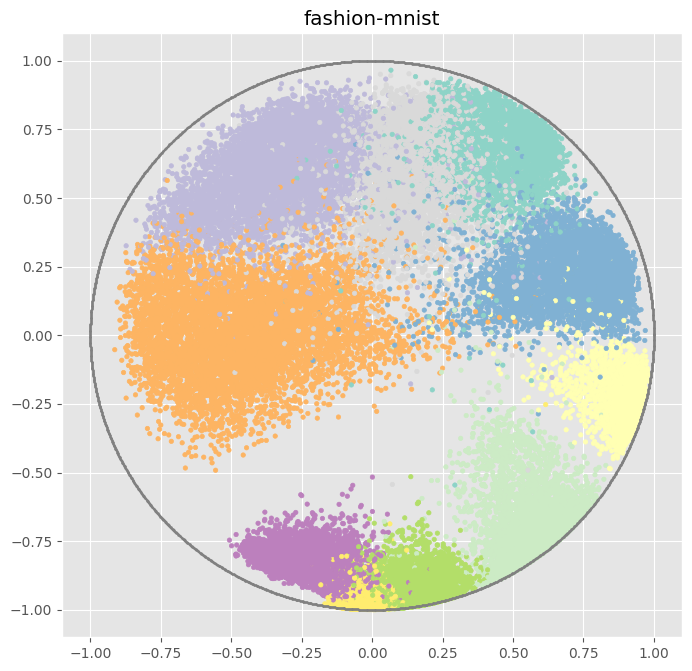}
    \end{subfigure}
    \hfill
    \begin{subfigure}
        \centering
        \includegraphics[width=0.15\textwidth,trim={1.5cm 1cm 0 1cm},clip]{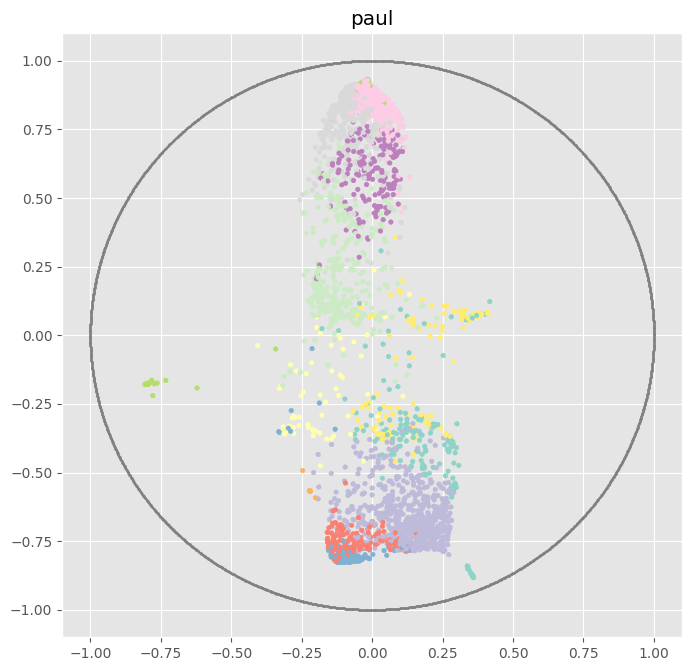}
    \end{subfigure} \\
    \caption{CIFAR, Fashion-MNIST, and Paul datasets Poincar\'{e} visualization. Each color correspond to one label class.}
    \label{fig:svm_data}
\end{figure}

\textbf{Results}
For each dataset, we fix a train-test split and run 5 times. In each run, we use a one-vs-all approach to train and use Platt scaling \citep{platt1999probabilistic} to give final predictions based on output probabilities. The average multiclass classification accuracy and macro F1 scores are reported. We compare our model (LSVMPP) against the Lorentz SVM (LSVM), the Euclidean SVM (ESVM) and a SVM with precomputed reference points based on the Poincar\'{e} ball (PSVM) \citep{chien2021highly}. See \Cref{implementation and hyperparams} for and a detailed discussion on implementation and respective optimal hyperparameters, \Cref{tab:real_dataset_result_partial} for average accuracy and F1 for selected datasets in \Cref{fig:svm_data}, and \Cref{app:svm_full_results} for experiment results on more datasets.

Although small fluctuations exist, our model outperforms other available methods in most of the simulated and real datasets. Note that in all cases, Euclidean SVM already has decent performances, unlike the unsatisfactory performances reported in \citep{chien2021highly} where ESVM models are mistakenly set to not have intercepts. The good performance by ESVM may be attributed to the powerful scaling method that universally applies to datasets well characterized by the log odds of positive samples, and this assumption may still hold on datasets with a latent hierarchical structure. It is also worth pointing out that the previous formulation of Lorentz SVM \citep{cho2019large} also uses Platt scaling on Lorentz SVM output but did not scale it using $\text{arcsinh}$ to morph it into a proxy of hyperbolic distance, which is more aligned with the Euclidean SVM output. Empirically we witness a marginal increase in average accuracy using Platt scaling with the output of $\text{arcsinh}$ instead of the raw alternatives (see \Cref{adaptation} for more details).

Lastly, notice that the previously leading model, PSVM, does not outperform ESVM in simple simulated cases. This matches our expectations in that the reference points in PSVM are separately learned using an unsupervised method (Hyperbolic Graham Search) to locate the midpoint between two convex hulls, independent of the hyperplane training process. This step is sensitive to the train test split and does not guarantee that PSVM could identify the max-margin separating hyperplane even in the training set. All other methods, if they converge, could indeed find the max-margin separating hyperplane in respective manifolds. 
\begin{table}[htbp]
    \centering
    \caption{Mean accuracy and macro F1 score of datasets in \Cref{fig:svm_data}.}
    \resizebox{\columnwidth}{!}{%
    \begin{tabular}{cc|cc} \toprule
         & algorithm & accuracy (\%) & F1 \\ 
         \midrule
        \multirow{4}{*}{CIFAR-10} & ESVM & 91.88 $\pm$ 0.00 & 0.9191 $\pm$ 0.00 \\ 
         & LSVM & 91.88 $\pm$ 0.00 & 0.9189 $\pm$ 0.00 \\
         & PSVM & 91.81 $\pm$ 0.00 & 0.9182 $\pm$ 0.00 \\
         & LSVMPP & \textbf{91.96 $\pm$ 0.00} & \textbf{0.9197 $\pm$ 0.00} \\
         \midrule
        \multirow{4}{*}{fashion-MNIST} & ESVM & 86.37 $\pm$ 0.00 & 0.8665 $\pm$ 0.00 \\ 
         & LSVM & 71.59 $\pm$ 0.07 & 0.6588 $\pm$ 0.08 \\ 
         & PSVM & 86.57 $\pm$ 0.00 & 0.8665 $\pm$ 0.00 \\ 
         & LSVMPP & \textbf{89.49 $\pm$ 0.00} & \textbf{0.8955 $\pm$ 0.00} \\
         \midrule 
         \multirow{4}{*}{paul} & ESVM & 55.05 $\pm$ 0.00 & 0.4073 $\pm$ 0.00 \\ 
         & LSVM & 58.36 $\pm$ 0.07 & 0.4517 $\pm$ 0.00 \\ 
         & PSVM & 55.25 $\pm$ 0.00 & 0.3802 $\pm$ 0.00 \\ 
         & LSVMPP & \textbf{62.64 $\pm$ 0.05} & \textbf{0.5024 $\pm$ 0.00} \\
         \bottomrule
    \end{tabular}%
    }
    \label{tab:real_dataset_result_partial}
\end{table}

\section{Discussions}
In this paper, we have analyzed the representation limitations and numerical stability of optimization for {two popular} models and a Euclidean parametrization of the hyperbolic space under a 64-bit arithmetic system. 
We note that however, while the Euclidean parameterization has no limitation for point representation, computation of {certain functions induced from hyperbolic spaces,} such as the hyperbolic distance between points, may still lead to numerical problems and we leave this for future study. Another potential future direction is to investigate the use of different accelerated gradient descent methods on these models.  A recent study \citep{hamilton2021no} has shown that the negative curvature inherent to the hyperbolic space can impede the acceleration of the gradient descent process in a Nestorov-like way, citing the exponential volume growth property as a cause for the loss of information from past gradients. It would be interesting to explore in the future whether the Euclidean parametrization proposed in our paper could overcome this obstruction and improve the convergence rate of optimization problems in hyperbolic spaces.

Code for reproducing our experiments is available at \href{https://github.com/yangshengaa/stable-hyperbolic}{https://github.com/yangshengaa/stable-hyperbolic}.

\section*{Acknowledgements} 
This work is partially supported by NSF grants CCF-2112665, CCF-2217058, and NIH grant RF1MH125317-01.


 
\bibliography{rep}
\bibliographystyle{icml2023}


\newpage
\appendix
\onecolumn
\section{Proofs}

\subsection{Proof of \Cref{prop:p radius}}\label{proof:p radius}
Note that $\mathrm{arccosh}(x)=\ln(x+\sqrt{x^2-1})$. This implies that
$$\mathrm{arccosh}\left(1+\frac{2\|x\|^2}{1-\|x\|^2}\right)=\mathrm{arccosh}\left(\frac{1+\|x\|^2}{1-\|x\|^2}\right)=\ln\left(\frac{1+\|x\|^2}{1-\|x\|^2}+\sqrt{\frac{(1+\|x\|^2)^2}{(1-\|x\|^2)^2}-1}\right)=\ln\left(\frac{1+\|x\|^2+2\|x\|}{1-\|x\|^2}\right),$$
and then $\ln\left(\frac{1+\|x\|^2+2\|x\|}{1-\|x\|^2}\right)=\ln\left(\frac{1+\|x\|}{1-\|x\|}\right)=\ln(1+\|x\|)-\ln(1-\|x\|)$.

By using the formula for the Poincar\'{e} distance between $0$ and $x$, we have that
\begin{align*}
    d_{\mathbb{D}^n}(0,x)&=\mathrm{arccosh}\left(1+ \frac{2\norm{x}^2}{1-\norm{x}^2}\right)\\
    &=\ln\left(1+\norm{x}\right)-\ln\left(1-\norm{x}\right)\\
    &=\ln(2-10^{-k})-\ln(10^{-k})\\
    &=\ln(10)k+\ln(2) +\ln(1-(2*10^k)^{-1})\\
    &=\ln(10)k+\ln(2) +O(10^{-k}).
\end{align*}

\subsection{Proof of \Cref{prop:l radius}}\label{proof:l radius}
By using the formula for the Lorentz distance between $\bar{0}$ and $x$, we have that
\begin{align*}
    d_{\mathbb{L}^n}(\bar{0},x)&=\mathrm{arccosh}(-[\bar{0},x])\\
    &=\mathrm{arccosh}(x_0)=\ln\left(x_0+\sqrt{x_0^2-1}\right)\\
    &=\ln(x_0)+\ln\left(1+\sqrt{1-1/x_0^2}\right)\\
    &=\ln(x_0)+\ln\left(2+\frac{1}{2x_0^2}+O(x_0^{-4})\right)\\
    &=\ln(10)k+\ln(2) +O(10^{-2k}).
\end{align*}

\subsection{Proof of \Cref{prop:same eta}}
Given a Riemannian isometry $\iota:M\rightarrow N$ between two manifolds, for any $p\in M$, one has that 
$$\iota\circ\exp_p=\exp_{\iota(p)}\circ D\iota_p,$$
where $D\iota_p:T_pM\rightarrow T_pN$ denotes the differential at $p$. 

Now, by the Riemannian chain rule for gradient, one has that 
\[\nabla_{\mathbb{D}^n} f(x)=\nabla_{\mathbb{D}^n}(g\circ\varphi)(x)=(D\varphi_x)^*\circ \nabla_{\mathbb{L}^n} g(y),\]
where $y=\varphi(x)$. Since $\varphi$ is an isometry, one has that $(D\varphi_x)^*=D\psi_{y}:T_{y}\mathbb{L}^n\rightarrow T_x\mathbb{D}^n$, where $\psi:=\varphi^{-1}$. In this way, $D\psi_{y}(-\eta\nabla_{\mathbb{L}^n} g(y))=-\eta \nabla_{\mathbb{D}^n} f(x)$ and thus
$$\psi(\exp_y(-\eta\nabla_{\mathbb{L}^n} g(y))= \exp_{\psi(y)}(D\psi_{y}(-\eta\nabla_{\mathbb{L}^n} g(y)))=\exp_{x}(-\eta\nabla_{\mathbb{D}^n} f(x)).$$

\subsection{Proof of \Cref{lm:gradient norm}}\label{proof:gradient norm}
The Poincar\'{e} gradient $\nabla_{\D^n} f$ at $y$ can be computed explicitly as follows:
\[  \nabla_{\D^n} f(x)=\lambda_x^{-2}\nabla f(x)= \frac{(1-\norm{x}^2)^2}{4}(\partial_1f(x),\ldots,\partial_nf(y))^\mathrm{T}.\]
Hence, 
\[\norm{\nabla_{\D^n} f(x)}= \Omega\left(\frac{(1-\norm{x}^2)^2}{4}\norm{\nabla f(x)}\right).\]

As for the Lorentz gradient, we apply the following lemma

\begin{lemma}\label{lm: differential}
For any $x\in\D^n$ and any $v=(v_1,\ldots,v_n)\in T_x\D^n$, by letting $y:=\varphi(x)\in\LL^n$, we can write down the image of $v$ under the differential $D\varphi$ explicitly as follows:

\begin{equation}\label{eq:p to l gradient}
    D\varphi(v) =\left(\sum_{i=1}^ny_iv_i(1+y_0),v_1(1+y_0)+\sum_iy_iv_i\cdot y_1,\ldots,v_n(1+y_0)+\sum_iy_iv_i\cdot y_n\right).
\end{equation}
\end{lemma}

\begin{proof}[Proof of \Cref{lm: differential}]
    The differential of $\psi:\LL^n\rightarrow\D^n$, the inverse of $\varphi$, was already derived in Equation (9) of \citep{wilson2018gradient}: for $w=(w_0,\ldots,w_n)\in T_y\LL^n$, and for any $1\leq i\leq n$
    \[(D\psi(w))_i=\frac{1}{1+y_0}\left(w_i-\frac{y_iw_0}{y_0+1}\right).\]
    We then simply verify that the map defined in \Cref{eq:p to l gradient} is the inverse of $D\psi$. Let $w$ be the right hand side of \Cref{eq:p to l gradient}. Then, for any $1\leq i\leq n$
    \begin{align*}
        (D\psi(w))_i&=\frac{1}{1+y_0}\left(w_i-\frac{y_iw_0}{y_0+1}\right)\\
        &=\frac{1}{1+y_0}\left(v_i(1+y_0)+\sum_jy_jv_j\cdot y_i-\frac{y_i}{y_0+1}\sum_{j=1}^ny_jv_j(1+y_0)\right)\\
        &=v_i.
    \end{align*}
    Hence, \Cref{eq:p to l gradient} holds.
\end{proof}

As the map $\varphi:\D^n\rightarrow\LL^n$ is an isometry, by the lemma above, the Lorentz gradient at $y$ can be written down explicitly: 
\begin{align*}
    &\nabla_{\LL^n} g(y) = D\varphi (\nabla_{\D^n}f(x))=D\varphi \left(\frac{(1-\norm{x}^2)^2}{4}\nabla f(x)\right)=D\varphi \left(\frac{1}{(1+y_0)^2}\nabla f(x)\right)\\
    &=\left(\frac{\sum_{i=1}^ny_i\partial_i f(x)}{1+y_0},\frac{\partial_1f(x)(1+y_0)+\sum_iy_i\partial_if(x)\cdot y_1}{(1+y_0)^2},\ldots,\frac{\partial_nf(x)(1+y_0)+\sum_iy_i\partial_if(x)\cdot y_n}{(1+y_0)^2}\right)^\mathrm{T}.
\end{align*}
Note that we use column vector representation in the above equation to indicate that the vector denotes a gradient. In the third equation, we used the fact that $1+y_0=1+\frac{1+\norm{x}^2}{1-\norm{x}^2}=\frac{2}{1-\norm{x}^2}$.
Then, 
\[\norm{\nabla_{\LL^n} g(y)} = \frac{\sqrt{2(\sum y_i\partial_if(x))^2+\sum(\partial_if(x))^2)}}{1+y_0}
    \leq \frac{\sqrt{2y_0^2-1}}{1+y_0}\norm{\nabla f(x)} =O(\norm{\nabla f(x)}).\]

\subsection{Proof of \Cref{thm:gd p vs l}}\label{proof:thm1}
Recall that in the Poincar\'{e} ball, the exponential map can be expressed as follows:
\[\exp_x(v)=x\oplus \tanh\left(\frac{\lambda_x}{2}\norm{v} \right)\frac{v}{\norm{v} }, \]
where $\lambda_x=\frac{2}{1-\norm{x} ^2}$ and
\begin{equation}\label{eq: poincare exp}
    x\oplus y =\frac{(1+2\langle x,y\rangle+\norm{y} ^2)x+(1-\norm{x} ^2)y}{1+2\langle x,y\rangle+\norm{x} ^2\norm{y} ^2}.
\end{equation}
When $x=(1-\delta,0,\ldots,0)$, $\lambda_x=\frac{2}{\delta(2-\delta)}$. 
We let $E:= \norm{\nabla  f(x)} $.
Then,
\begin{align*}
    \exp_x(-\eta\nabla_{\D^n} f(x))&=x\oplus \tanh\left(\frac{\lambda_x\eta}{2}\norm{\nabla_{\D^n} f(x)} \right)\frac{-\nabla_{\D^n} f(x)}{\norm{\nabla_{\D^n} f(x)} }\\
    &=x\oplus \tanh\left(\frac{\lambda_x\eta}{2}\norm{\nabla_{\D^n} f(x)} \right)\frac{-\nabla_{\D^n} f(x)}{\norm{\nabla_{\D^n} f(x)} }\\
    &=x\oplus \tanh\left(\frac{\eta}{2\lambda_x}\norm{\nabla  f(x)} \right)\frac{-\nabla  f(x)}{\norm{\nabla  f(x)} }\\
    &=x\oplus \tanh\left(\frac{\eta E \delta(2-\delta)}{4}\right)\frac{-\nabla  f(x)}{\norm{\nabla  f(x)} }\\
    &=x\oplus \left(\frac{ 2-\delta}{4}\eta E\delta+O((\eta E\delta)^3)\right)\frac{-\nabla  f(x)}{\norm{\nabla  f(x)} }.
\end{align*}
We use the Taylor expansion of $\tanh$ at $0$ in the last equation. Now, using \Cref{eq: poincare exp}, we have that
\begin{align*}
    &x\oplus \left(\frac{2-\delta}{4}\eta E \delta+O((\eta E \delta)^3)\right)\frac{-\nabla  f(x)}{\norm{\nabla  f(x)}}\\
    =&\frac{(1-2(1-\delta)(\frac{2-\delta}{4}\eta E\delta+O((\eta E\delta)^3))+O((\eta E\delta)^2))(1-\delta,0,\ldots,0)+\delta(2-\delta)(\frac{2-\delta}{4}\eta E\delta+O((\eta E\delta)^2))(1,0,\ldots,0)}{1-2(1-\delta)(\frac{2-\delta}{4}\eta E\delta+O((\eta E\delta)^3))+(1-\delta)^2O((\eta E\delta)^2)}\\
    =&\frac{\big(1-2(1-\delta)\frac{2-\delta}{4}\eta E\delta+O((\eta E\delta)^2))(1-\delta)+O(\eta E\delta ^2),0,\ldots,0\big)}{1-2(1-\delta)\frac{2-\delta}{4}\eta E\delta+O((\eta E\delta)^2)}\\
    =&(1-\delta +O(\eta E\delta^2),0,\ldots,0).
\end{align*}

As for the Lorentz model, we have that
\begin{align*}
    &\exp_y(-\eta\nabla_{\LL^n}g(y))=\cosh(\eta\norm{\nabla_{\LL^n}g(y)}_{\LL^n})y+\sinh(\eta\norm{\nabla_{\LL^n}g(y)}_{\LL^n})\frac{-\nabla_{\LL^n}g(y)}{\norm{\nabla_{\LL^n}g(y)}_{\LL^n}}\\
    =&\cosh(\eta\norm{\nabla_{\D^n}f(x)}_{\D^n})y+\sinh(\eta\norm{\nabla_{\D^n}f(x)}_{\D^n})\frac{-\nabla_{\LL^n}g(y)}{\norm{\nabla_{\D^n}f(x)}_{\D^n}}\\
    =&\cosh(\frac{(2-\delta)}{2}\eta E\delta)\big(\frac{2-2\delta+\delta^2}{\delta(2-\delta)},\frac{2-2\delta}{\delta(2-\delta)},0,\ldots,0\big)\\
    -&\frac{\sinh(\frac{(2-\delta)}{2}\eta E\delta)}{\frac{(2-\delta)}{2} E\delta}\big((1-\delta)\partial_1f(x),\frac{2-2\delta+\delta^2}{2}\partial_1 f(x),0,\ldots,0\big)\\
    =&(1+O((\eta E\delta)^2))\big(\frac{2-2\delta+\delta^2}{\delta(2-\delta)},\frac{2-2\delta}{\delta(2-\delta)},0,\ldots,0\big)+(\eta + O(\eta^2(E\delta)^3))\big((1-\delta)E,\frac{2-2\delta+\delta^2}{2}E,0,\ldots,0\big)\\
    =&\big(\frac{1}{\delta}-\frac{1}{2}+\eta E+O(\eta E\delta),\frac{1}{\delta}-\frac{1}{2}+\eta E+O(\eta E\delta),0,\ldots,0\big)
\end{align*}

Finally, we substitute $\delta$ with $10^{-k}$ to obtain 
\begin{align*}
    \exp_x(-\eta\nabla_{\D^n} f(x))&= (1-10^{-k} +O(\eta E10^{-2k}),0,\ldots,0), \\
    \exp_y(-\eta\nabla_{\LL^n} g(y))&= \Large(10^k-\frac{1}{2}+\eta E+O(\eta E10^{-k}), 10^k-\frac{1}{2}+\eta E+O(\eta E10^{-k}),0,\ldots,0\large).
\end{align*} 

\subsection{Proof of \Cref{thm:euclidean}}\label{proof:thm2}
Similarly as in the proof of \Cref{proof:thm1}, we denote $E := \norm{\nabla f(x)}$. We first calculate the gradient $\nabla h$. We let $\mathrm{J}_{F_D}$ denote the Jacobian of $F_D$. Then, the differential of $h$ can be computed via the chain rule
\[Dh(z)=Df(x)\cdot J_{F_D}(z)\]
Therefore, as for the gradient, we have that
\begin{equation}\label{eq:gradient of composition}
    \nabla h(z) = (Dh(z))^\mathrm{T}= (\mathrm{J}_{F_D}(z))^\mathrm{T}\nabla f(x),
\end{equation}
where we use the fact that $\nabla f(x)= (Df(x))^\mathrm{T}$.

\begin{remark}
    Of course, one also can apply the chain rule formula for Riemannian gradient directly as follows 
    \[\nabla h(z)= (D{F_D}(z))^*\cdot \nabla_{\D^n}f(x),\]
    where $(D{F_D}(z))^*$ denotes the adjoint of the differential map $D{F_D}(z):T_z\R^n\rightarrow T_x\D^n$, which is represented by a matrix. Given that the metric on $T_x\D^n$ is $\lambda_x^2h_x$ where $h_x$ denotes the Euclidean metric, we have that the matrix representation of $(D{F_D}(z))^*$ is $(\mathrm{J}_{F_D}(z))^\mathrm{T}\lambda_x^2$. Hence 
    \[\nabla h(z) =(\mathrm{J}_{F_D}(z))^\mathrm{T}\lambda_x^2\cdot \underbrace{\lambda_x^{-2}\nabla f(x)}_{A}=(\mathrm{J}_{F_D}(z))^\mathrm{T}\nabla f(x).\]
    This agrees with \Cref{eq:gradient of composition}.
\cite{guo2022clipped} claimed that functions of the form of $h$, i.e., the form $\R^n\xrightarrow[]{\exp_0}\D^n\rightarrow\R$, will suffer severe vanishing gradient issue since the term $A$ has a coefficient $\lambda^{-2}_x=\frac{(1-\norm{x}^2)^2}{4}$ vanishing as $x$ is close to the boundary of $\D^n$. However, as one can see clearly from the above formula, the coefficient $\frac{(1-\norm{x}^2)^2}{4}$ will be canceled when considering the map $F_D$ (or equivalently speaking, the exponential map $\exp_0$). Hence, in order to correctly analyze the magnitude of $\nabla h$, one needs to compute the Jacobian matrix which is what we do next in a special case.
\end{remark}

Let $w:=\mathrm{arctanh}(1-\delta)$. Then, $z=(2w,0,\ldots,0)$ and the Jacobian of $F_D:\R^n\rightarrow\D^n$ at $z=(2w,0,\ldots,0)$ is
\[\mathrm{J}_{F_D}(z) = \frac{\tanh(w)}{2w}*\mathrm{I}_{n\times n}+
\begin{bmatrix}
            \frac{1}{2}(1-\mathrm{tanh}^2(w)-\frac{\tanh(w)}{w}) & 0 &\dotsc &0\\
            0 & 0 &\dotsc &0\\
            \vdots & \vdots & \ddots & \vdots\\
            0 & 0 & \dotsc & 0
        \end{bmatrix}.\]
Hence, 
\begin{align*}
   \nabla h(z)&=(\mathrm{J}_{F_D}(z))^\mathrm{T}\nabla f(x)=\left( \frac{\tanh(w)}{2w}*\mathrm{I}_{n\times n}+\begin{bmatrix}
            \frac{1}{2}(1-\mathrm{tanh}^2(w)-\frac{\tanh(w)}{w}) & 0 &\dotsc &0\\
            0 & 0 &\dotsc &0\\
            \vdots & \vdots & \ddots & \vdots\\
            0 & 0 & \dotsc & 0
        \end{bmatrix}\right)\begin{bmatrix}
            \partial_1f(x)\\
            0 \\
            \vdots \\
            0 
        \end{bmatrix}\\
        &=\left[\frac{\partial_1f(x)}{2}(1-\mathrm{tanh}^2(w)),0,\ldots,0\right]^\mathrm{T}.
\end{align*}

Therefore, the first component of $z-\eta\nabla h(z)$ is such that
\begin{align*}
    &2w + \frac{\eta E}{2}(1-\mathrm{tanh}^2(w))\\
    =&2w+\frac{\eta E}{2}(1-(1-\delta)^2)\\
    =&\ln\left(\frac{2}{\delta}\right)-\frac{\delta}{2}+O(\delta^2)+\frac{\eta E}{2}(2\delta-\delta^2)\\
    =&\ln\left(\frac{2}{\delta}\right)+({\eta E}-1/2)\delta+O(\delta^2).
\end{align*}
Note that in the second equation we use the expansion that
\[\mathrm{arctanh}(1-\delta)=\frac{1}{2}\ln\left(\frac{2}{\delta}\right)-\frac{\delta}{4}+O(\delta^2).\]

Finally, we substitute $\delta$ with $10^{-k}$ to obtain 
\[z-\eta\nabla h(z)=\ln(10)k+\ln(2)+({\eta \norm{\nabla f(x)}}-1/2)10^{-k}+O(10^{-2k}).\]

\subsection{Derivation of \Cref{eq:vector parametrization}}\label{proof:derivation}
In the hyperbolic space $\mathbb{H}^n$, any two points can be connected by a unique geodesic. Using the Lorentz model,  for any $x,y\in\LL^n$, the parallel transport map along the geodesic connecting $x$ and $y$ can be written down explicitly as follows: for $v\in T_x\LL^n$, one has that
\[\text{PT}_{x \mapsto y}(v) = v + \frac{[y, v]}{1 - [x,y]}(x + y).\]
As usual we define $\bar{0} := [1, 0, ..., 0] \in \mathbb{R}^{n + 1}$ to be the origin of the Lorentz model. Then
\begin{align*}
    w&=\text{PT}_{\bar{0}\mapsto p}(\bar{z})=\bar{z}+\frac{[p,\bar{z}]}{1-[\bar{0},p]}(\bar{0}+p)\\
    &=(0,z)+\frac{\sinh(a)\norm{z}}{1+\cosh(a)}\left(1+\cosh(a),\sinh(a)\frac{z}{\norm{z}}\right)\\
    &=\left(\sinh(a)\norm{z},z+\frac{\sinh^2(a)}{1+\cosh(a)}z\right)\\
    &=\left(\sinh(a)\norm{z},z+\frac{\cosh^2(a)-1}{1+\cosh(a)}z\right)\\
    &=\left(\sinh(a)\norm{z},\cosh(a)z\right).
\end{align*}

\section{Tree Optimization}\label{sec:tree opt}

In this section, we provide more details on the tree optimization experiments.

\subsection{Synthetic Data Generating Processing}\label{subsec:syn_tree_dgp}
We prefer to initialize, center, and normalize trees with Euclidean coordinates in a prescribed bounded region, serving as a data preprocessing procedure. Here we fit every simulated tree to be within a unit square. This ensures numerical stability in the initial exponential map into either $\D^n$ or $\LL^n$, and empirically has lower embedding distortion than the same ones with larger scales. The reason is that if the data is not centered or of a large scale, the subsequent exponential map brings features to much more distorted positions and poses challenges in the training steps.

\subsection{Full Results}
The Euclidean initialization of all tested trees are visualized in \Cref{tab:sim_tree_vis_full}. In  \Cref{tab:average_distortion_full}, we present all embedding results in distortion $\delta$, max distortion $\delta^{\max}$, and embedding diameters $d$.

\begin{figure*}[htbp]
    \centering
    \begin{subfigure}
        \centering
        \includegraphics[width=0.23\textwidth]{images/sim_tree_vis_404.png}
    \end{subfigure} 
    \hfill
    \begin{subfigure}
        \centering
        \includegraphics[width=0.23\textwidth]{images/sim_tree_vis_407.png}
    \end{subfigure} 
    \hfill
    \begin{subfigure}
        \centering
        \includegraphics[width=0.23\textwidth]{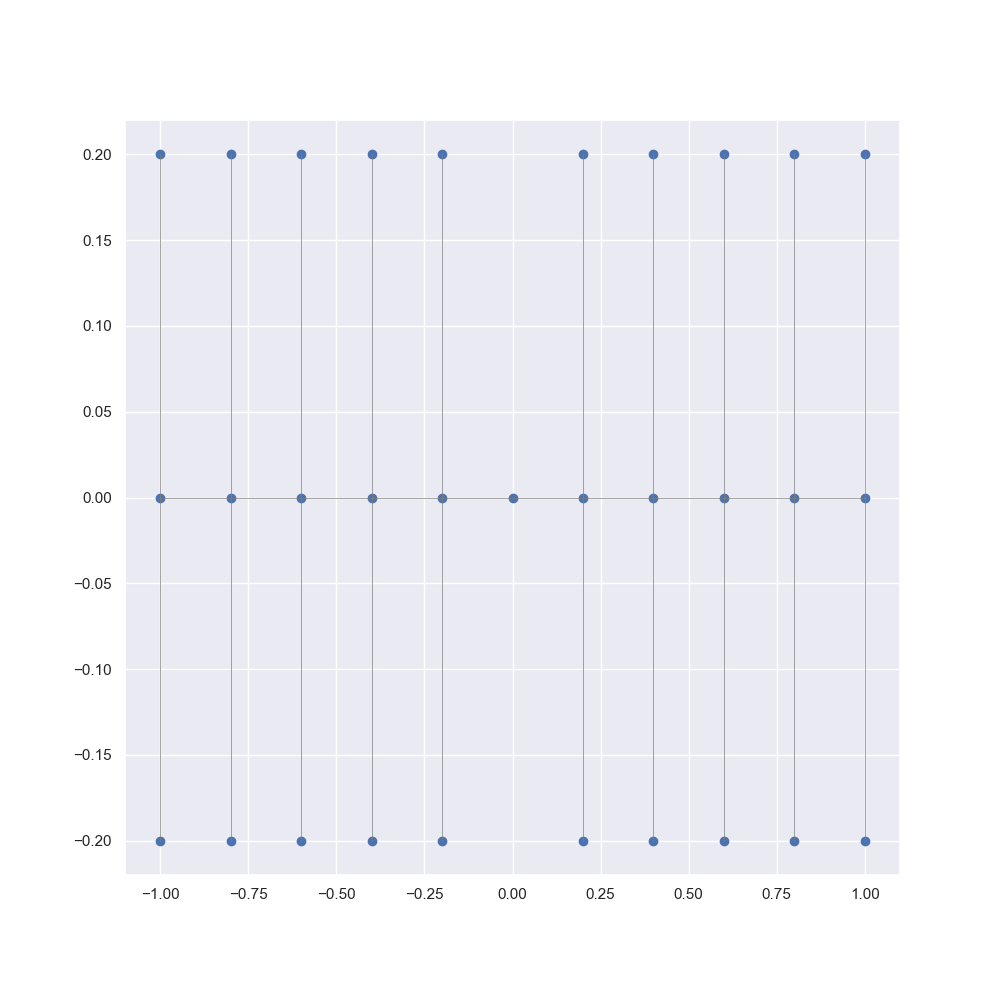}
    \end{subfigure} 
    \hfill
    \begin{subfigure}
        \centering
        \includegraphics[width=0.23\textwidth]{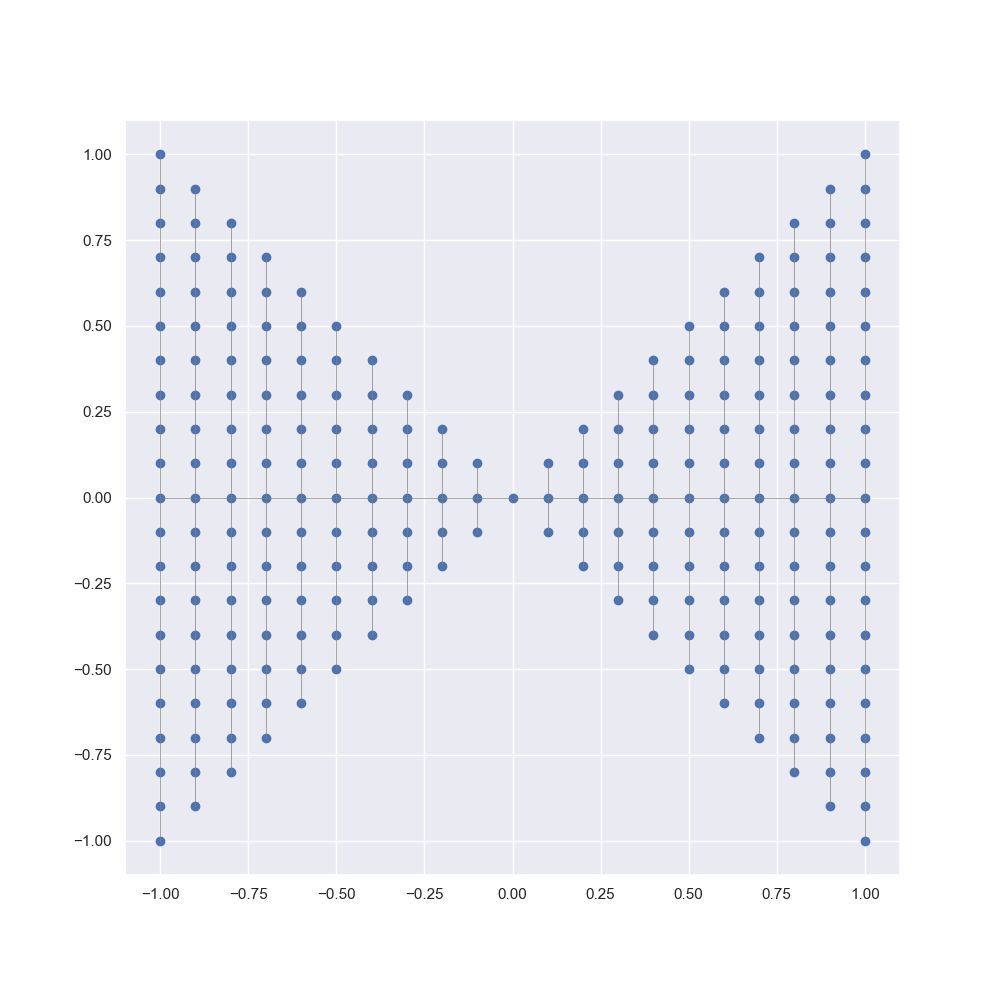}
    \end{subfigure} \\

    \begin{subfigure}
        \centering
        \includegraphics[width=0.23\textwidth]{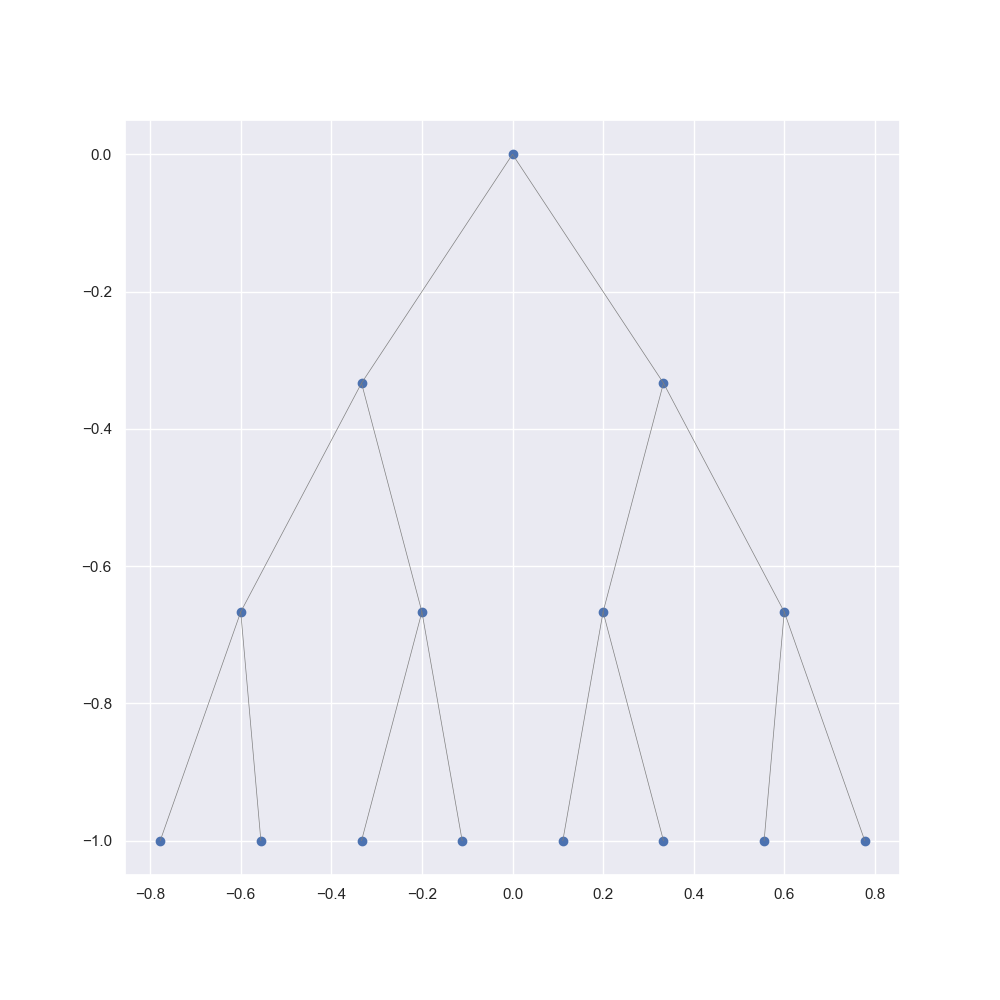}
    \end{subfigure} 
    \hfill
    \begin{subfigure}
        \centering
        \includegraphics[width=0.23\textwidth]{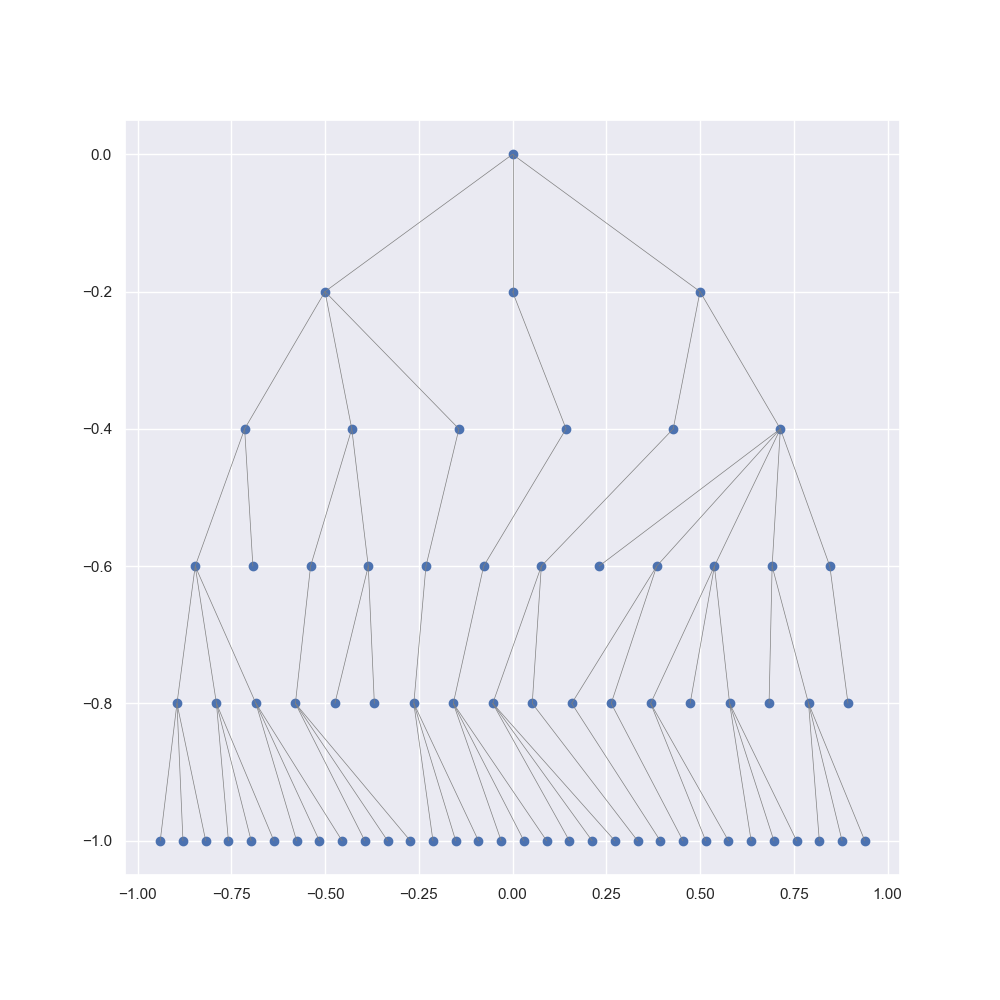}
    \end{subfigure} 
    \hfill
    \begin{subfigure}
        \centering
        \includegraphics[width=0.23\textwidth]{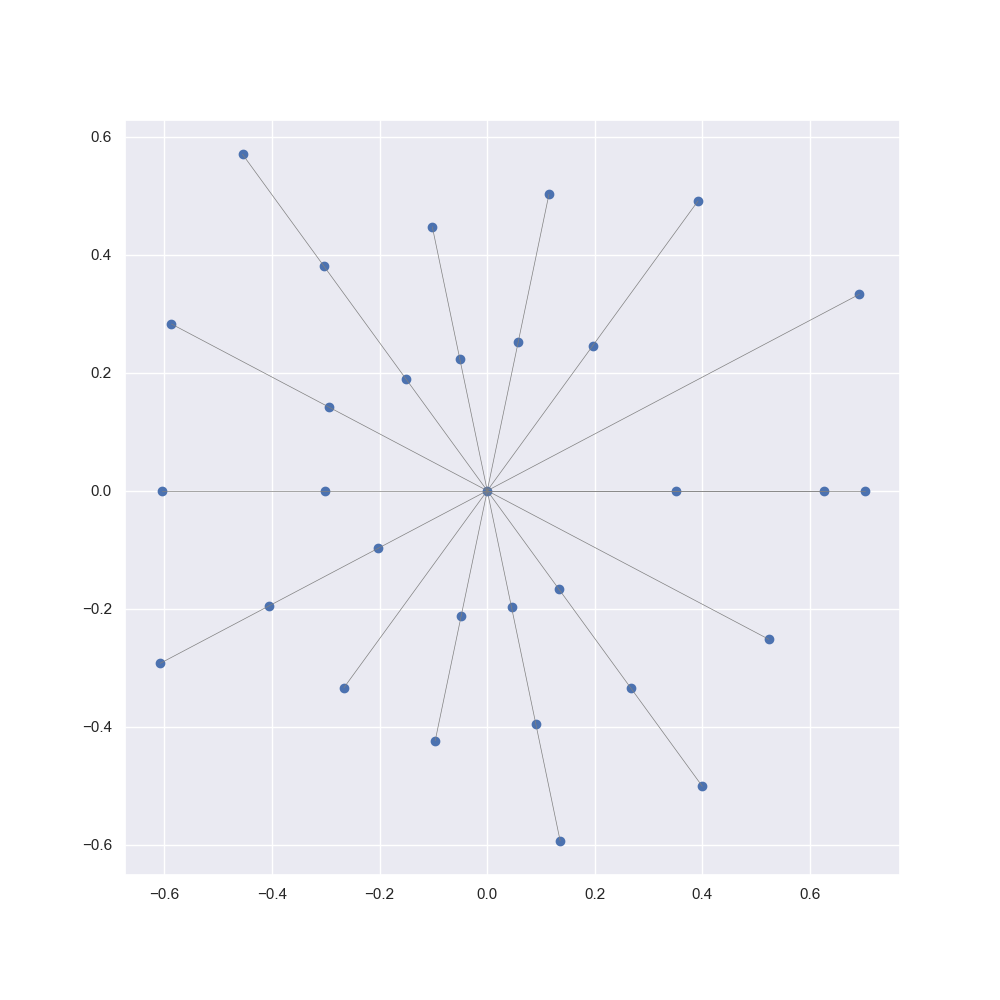}
    \end{subfigure} 
    \hfill
    \begin{subfigure}
        \centering
        \includegraphics[width=0.23\textwidth]{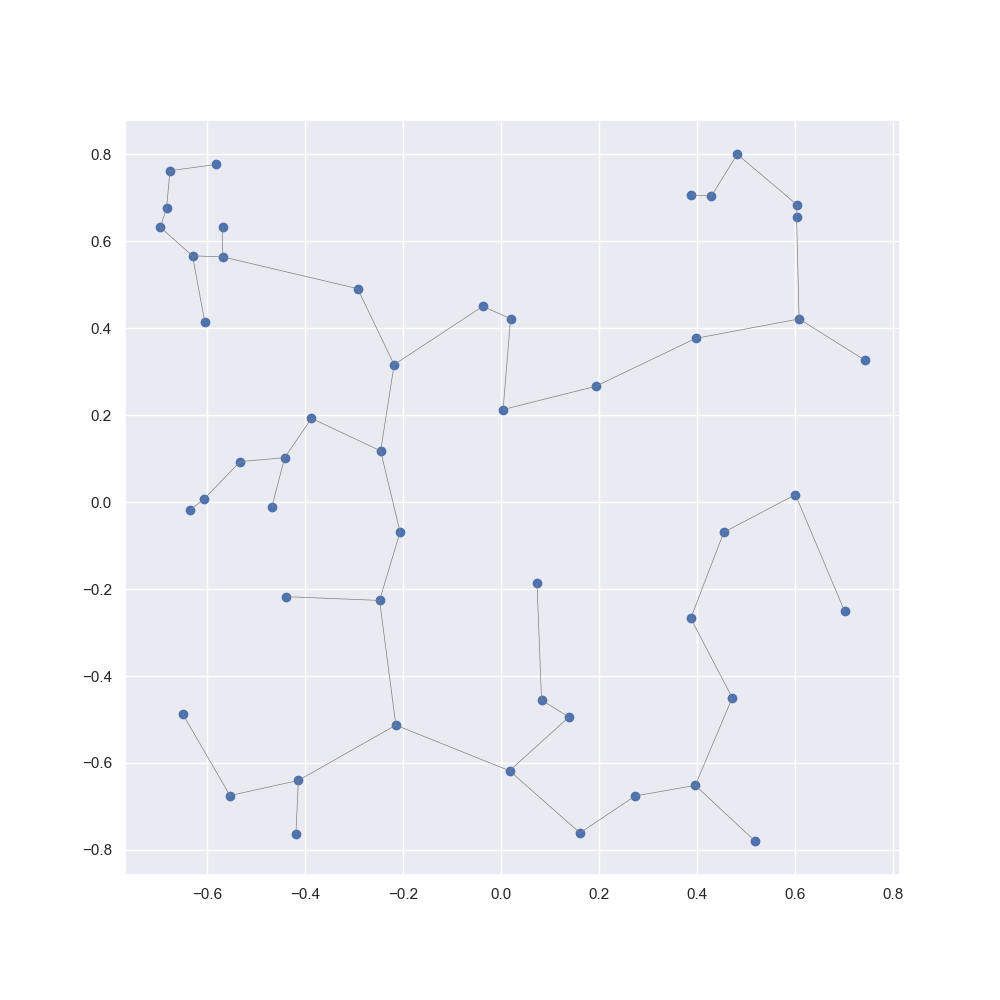}
    \end{subfigure} \\

    \caption{Simulated Trees in $\mathbb{R}^2$. We label figures by Tree 1 - 4 (top row) and Tree 5 - 8 (bottom row). The 2D coordinates of each node are features and pairwise distances computed through shortest path distance on the connected graph.}
    \label{tab:sim_tree_vis_full}
\end{figure*}
\begin{table*}[htbp]
    \centering
    \caption{Average distortion $\delta$, max distortion $\delta^{\max}$, and diameter $d$ by Dataset created in \Cref{tab:sim_tree_vis_full} and Manifold} 
    \begin{tabular}{cc|ccc||cc|ccc}\toprule
        tree & manifold & $\delta$ & $\delta^{\max}$ & $d$ & tree & manifold & $\delta$ & $\delta^{\max}$ & $d$ \\
        \midrule
        \multirow{4}{*}{1} & raw            & 1.1954 & 10.6062 & 2.8284 & 
        \multirow{4}{*}{5} & raw            & 1.3349 & 12.7279 & 1.4142 \\
                           & $\mathbb{D}^2$ & 1.0462 & 4.0546 & 4.8740 & 
                           & $\mathbb{D}^2$ & 1.0176 & 2.0321 & 4.2523 \\
                           & $\mathbb{L}^2$ & 1.0176 & 2.4511 & 10.1827 & 
                           & $\mathbb{L}^2$ & 1.0029 & 1.2829 & 7.3850 \\
                           & $\mathbb{E}^2$ & \textbf{1.0157} & \textbf{2.3158} & \textbf{10.9019} & 
                           & $\mathbb{E}^2$ & \textbf{1.0019} & \textbf{1.1997} & \textbf{8.8504} \\
                           \midrule 
        \multirow{4}{*}{2} & raw            & 1.1186 & 14.1421 & 2.2361 & 
        \multirow{4}{*}{6} & raw            & 1.5080 & 88.8552 & 1.4142 \\
                           & $\mathbb{D}^2$ & 1.0589 & 7.1435 & 4.9757 & 
                           & $\mathbb{D}^2$ & 1.1243 & 50.4524 & 4.3168 \\
                           & $\mathbb{L}^2$ & 1.0180 & 3.2803 & 10.6719 & 
                           & $\mathbb{L}^2$ & 1.0392 & \textbf{21.9948} & 8.9204 \\
                           & $\mathbb{E}^2$ & \textbf{1.0134} & \textbf{2.7408} & \textbf{11.4504} &
                           & $\mathbb{E}^2$ & \textbf{1.0269} & 22.3743 & \textbf{9.8190} \\
                           \midrule 
        \multirow{4}{*}{3} & raw            & 1.0511 & 3.0000 & 2.0396 & 
        \multirow{4}{*}{7} & raw            & 1.2108 & 17.7132 & 0.8321 \\
                           & $\mathbb{D}^2$ & 1.0321 & 3.3830 & 6.3564 & 
                           & $\mathbb{D}^2$ & 1.2145 & 8.6454 & 5.0657 \\
                           & $\mathbb{L}^2$ & 1.0190 & 2.5506 & 9.6830 & 
                           & $\mathbb{L}^2$ & 1.0166 & 3.7655 & 19.1703 \\
                           & $\mathbb{E}^2$ & \textbf{1.0174} & \textbf{2.4350} & \textbf{10.2328} & 
                           & $\mathbb{E}^2$ & \textbf{1.0100} & \textbf{1.7825} & \textbf{19.8481} \\
                           \midrule
        \multirow{4}{*}{4} & raw            & 1.1539 & 19.0000 & 2.8284 & 
        \multirow{4}{*}{8} & raw            & 1.1421 & 10.7807 & 0.9220 \\
                           & $\mathbb{D}^2$ & 1.0754 & 14.1074 & 5.0453 & 
                           & $\mathbb{D}^2$ & 1.0324 & 9.7552 & 4.5670 \\
                           & $\mathbb{L}^2$ & 1.0421 & 8.2933 & 10.4028 & 
                           & $\mathbb{L}^2$ & 1.0157 & 9.2790 & 8.8324 \\
                           & $\mathbb{E}^2$ & \textbf{1.0292} & \textbf{6.8504} & \textbf{12.7596} & 
                           & $\mathbb{E}^2$ & \textbf{1.0148} & \textbf{8.4500} & \textbf{9.1712} \\
                           \bottomrule
    
    \end{tabular}%
    \label{tab:average_distortion_full}
\end{table*}

The hierarchy that the Euclidean parametrization is comparable with and slightly better than the Lorentz model and that the Lorentz model is better than the Poincar\'{e} ball persists in all tested datasets. We also visualize all embeddings at their final training stage in \Cref{tab:sim_tree_vis_emb_part1} and \Cref{tab:sim_tree_vis_emb_part2}. Note that Poincar\'{e} embeddings in general are more contracted than their alternatives, further demonstrating its optimization disadvantage. 

\begin{figure}[htbp]
    \centering
    \begin{subfigure}
        \centering
        \includegraphics[width=0.3\textwidth]{images/tree_404_poincare.png}
    \end{subfigure} 
    \hfill 
    \begin{subfigure}
        \centering
        \includegraphics[width=0.3\textwidth]{images/tree_404_lorentz.png}
    \end{subfigure} 
    \hfill 
    \begin{subfigure}
        \centering
        \includegraphics[width=0.3\textwidth]{images/tree_404_euclidean.png}
    \end{subfigure}  \\

    \begin{subfigure}
        \centering
        \includegraphics[width=0.3\textwidth]{images/tree_407_poincare.png}
    \end{subfigure} 
    \hfill 
    \begin{subfigure}
        \centering
        \includegraphics[width=0.3\textwidth]{images/tree_407_lorentz.png}
    \end{subfigure} 
    \hfill 
    \begin{subfigure}
        \centering
        \includegraphics[width=0.3\textwidth]{images/tree_407_euclidean.png}
    \end{subfigure}  \\

    \begin{subfigure}
        \centering
        \includegraphics[width=0.3\textwidth]{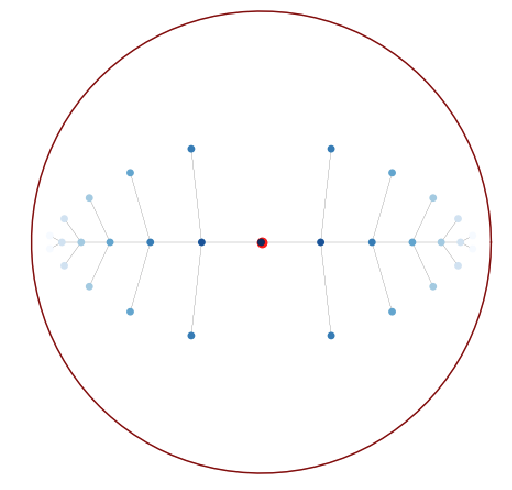}
    \end{subfigure} 
    \hfill 
    \begin{subfigure}
        \centering
        \includegraphics[width=0.3\textwidth]{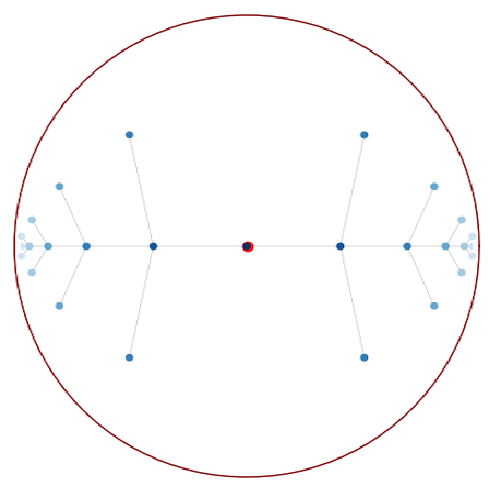}
    \end{subfigure} 
    \hfill 
    \begin{subfigure}
        \centering
        \includegraphics[width=0.3\textwidth]{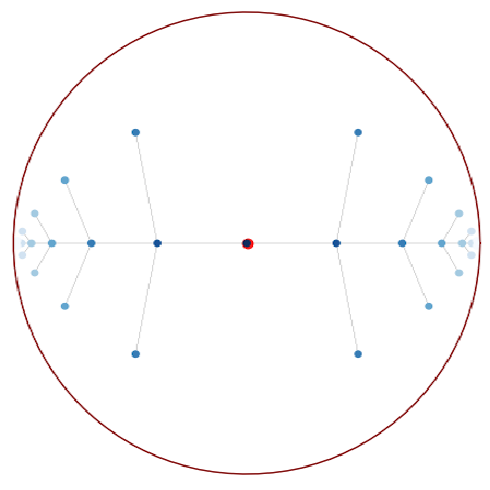}
    \end{subfigure}  \\

    \begin{subfigure}
        \centering
        \includegraphics[width=0.3\textwidth]{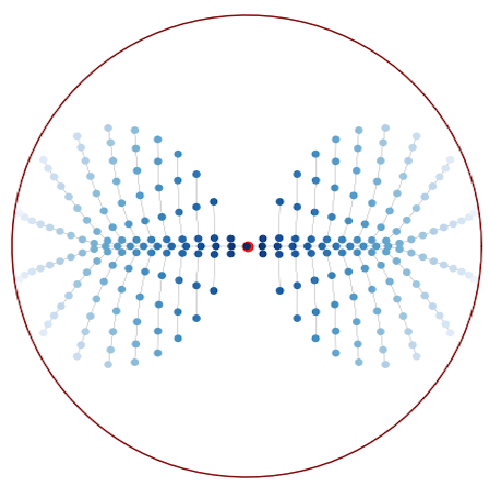}
    \end{subfigure} 
    \hfill 
    \begin{subfigure}
        \centering
        \includegraphics[width=0.3\textwidth]{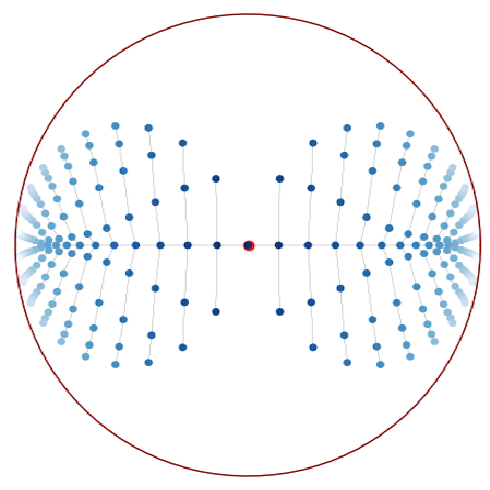}
    \end{subfigure} 
    \hfill 
    \begin{subfigure}
        \centering
        \includegraphics[width=0.3\textwidth]{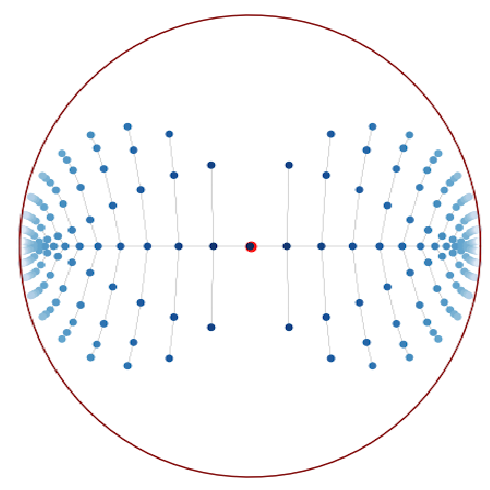}
    \end{subfigure}  \\
    \caption{Visualized Embeddings (Tree 1 - 4, from top to bottom row) at the final training epoch (3000). Each column from left to right correspond to Poincar\'{e}, Lorentz, and Euclidean parametrization.}
    \label{tab:sim_tree_vis_emb_part1}
\end{figure}
\begin{figure}[htbp]
    \centering
    \begin{subfigure}
        \centering
        \includegraphics[width=0.3\textwidth, height=0.3\textwidth]{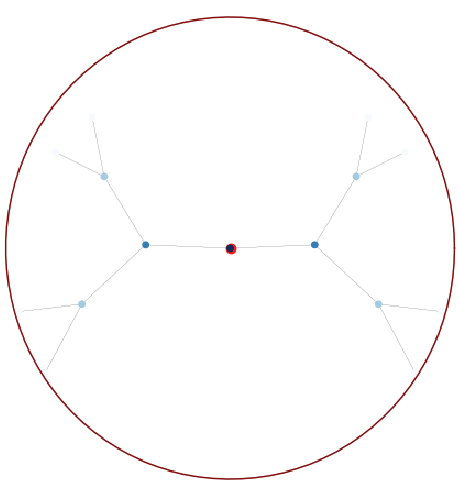}
    \end{subfigure} 
    \hfill 
    \begin{subfigure}
        \centering
        \includegraphics[width=0.3\textwidth, height=0.3\textwidth]{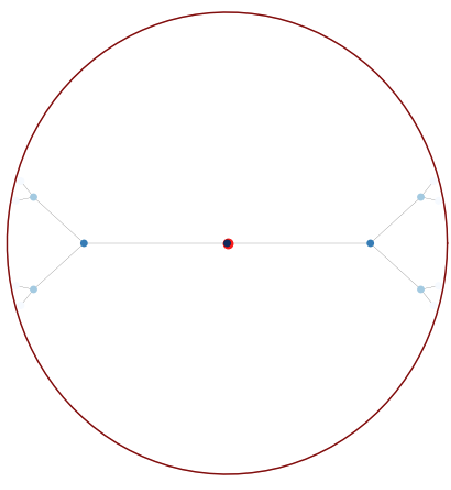}
    \end{subfigure} 
    \hfill 
    \begin{subfigure}
        \centering
        \includegraphics[width=0.3\textwidth, height=0.3\textwidth]{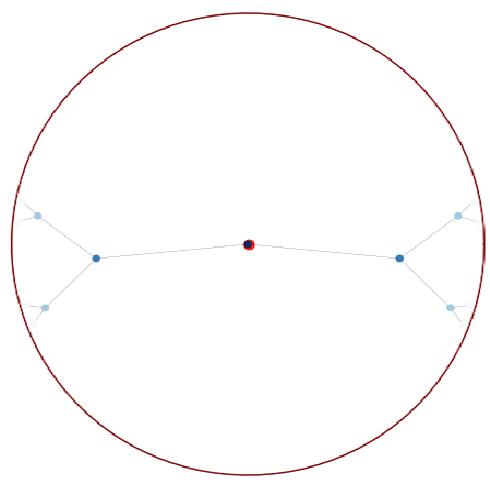}
    \end{subfigure}  \\

    \begin{subfigure}
        \centering
        \includegraphics[width=0.3\textwidth, height=0.3\textwidth]{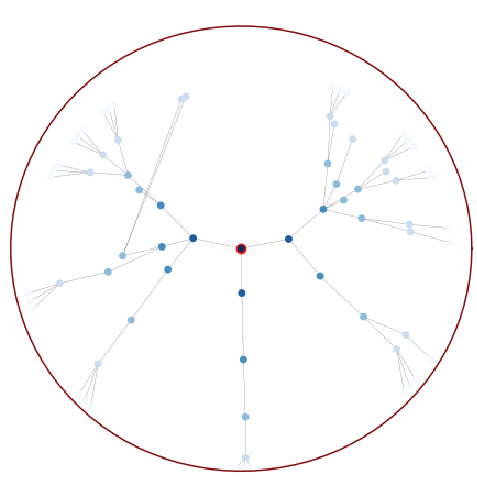}
    \end{subfigure} \hfill \begin{subfigure}
        \centering
        \includegraphics[width=0.3\textwidth, height=0.3\textwidth]{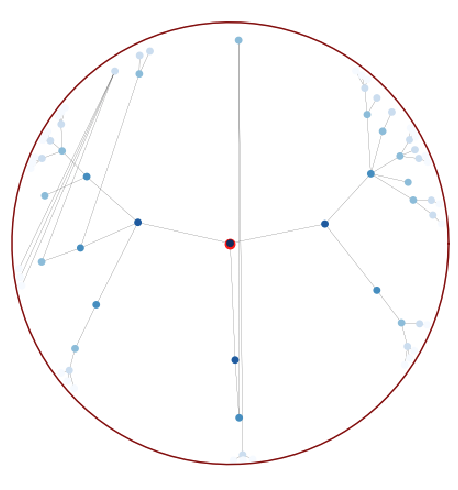}
    \end{subfigure} \hfill \begin{subfigure}
        \centering
        \includegraphics[width=0.3\textwidth, height=0.3\textwidth]{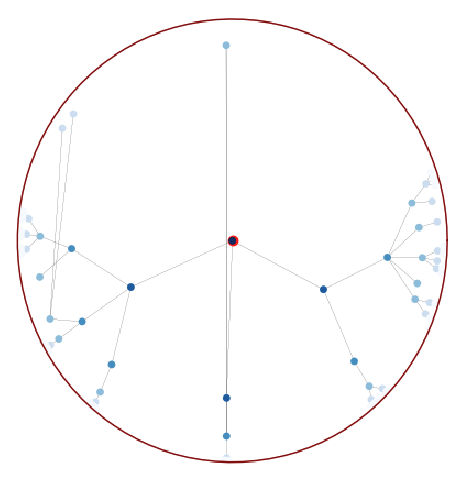}
    \end{subfigure}  \\

    \begin{subfigure}
        \centering
        \includegraphics[width=0.3\textwidth]{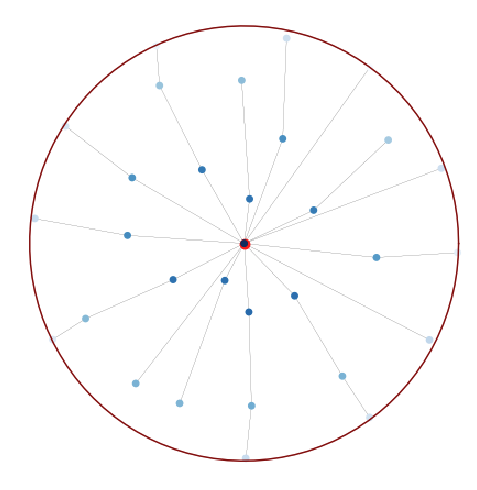}
    \end{subfigure} \hfill \begin{subfigure}
        \centering
        \includegraphics[width=0.3\textwidth]{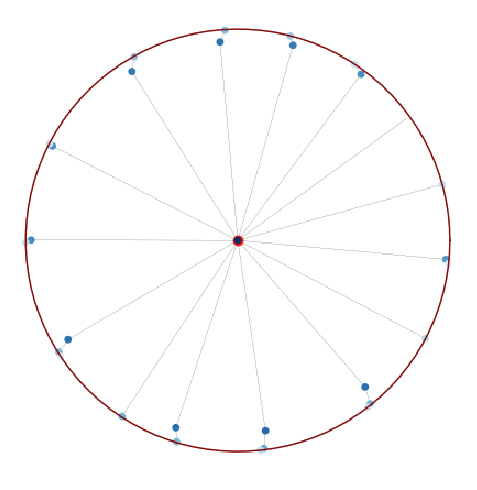}
    \end{subfigure} \hfill \begin{subfigure}
        \centering
        \includegraphics[width=0.3\textwidth]{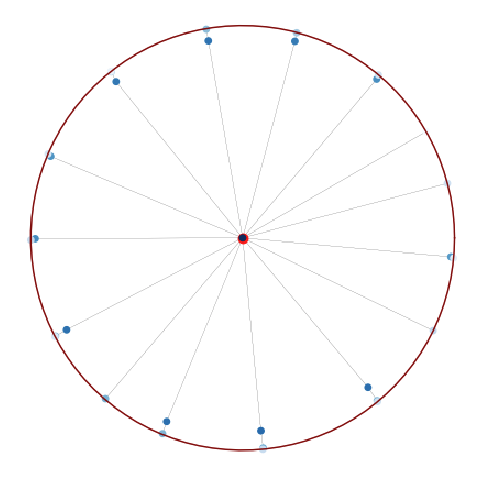}
    \end{subfigure}  \\

    \begin{subfigure}
        \centering
        \includegraphics[width=0.3\textwidth]{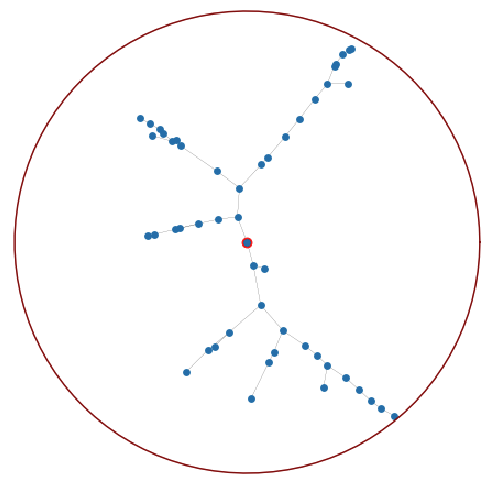}
    \end{subfigure} 
    \hfill 
    \begin{subfigure}
        \centering
        \includegraphics[width=0.3\textwidth]{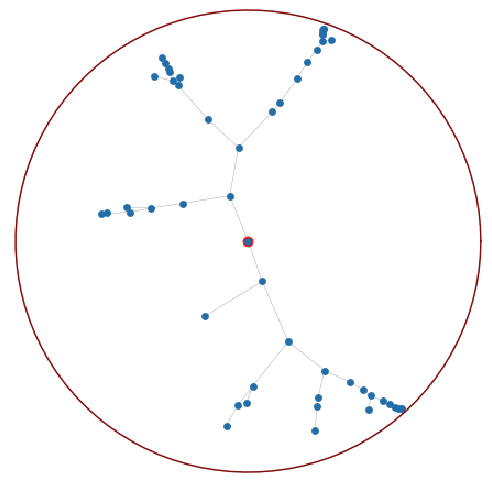}
    \end{subfigure}
    \hfill 
    \begin{subfigure}
        \centering
        \includegraphics[width=0.3\textwidth]{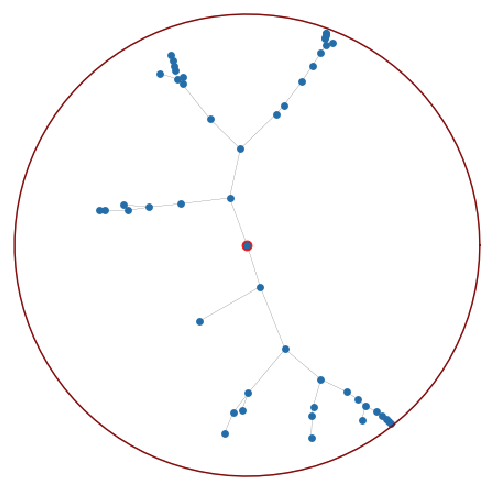}
    \end{subfigure}  \\
         
    \caption{Visualized Embeddings (Tree 5 - 8, from top to bottom row) at the final training epoch (3000). Each column from left to right correspond to Poincar\'{e}, Lorentz, and  Euclidean parametrization.}
    \label{tab:sim_tree_vis_emb_part2}
\end{figure}

\section{More on SVM}
In this section, we provide more details of our experiments on SVM.

\subsection{Overview of SVM in Binary Classification}

Support Vector Machine (SVM) is a machine learning task for both classification and regression by training a separating hyperplane. We first provide a brief overview of the Euclidean SVM and then introduce its variants in the hyperbolic space \citep{cho2019large,chien2021highly} using $l_2$ loss and a hinge-loss penalty. Readers may consider using $l_1$ loss or squared-hinge penalty, and other combinations as well. 

\paragraph{ESVM}
 In the classification setting, let $\{((x_i), (y_i)) \}_{i=1}^n$ be the training set and $x_i \in \mathbb{R}^n, y_i \in \{1, -1\}, \forall i \in \{1, ..., n\}$.
To obtain the separating hyperplane, in the soft-margin Euclidean SVM (ESVM), we minimize the margin between the hyperplane and the support vectors
$$
\underset{w \in \mathbb{R}^n}{\min} \; \frac{1}{2} \norm{w}_2^2 + C\sum_{i = 1}^n l(y_i \langle w, x_i \rangle)
$$
where $l$ is a hinge loss $l(z) = \max(0, 1 - z)$, and $C$ controls the tolerance of misclassification. 

\paragraph{LSVM}
Given a hyperbolic classification dataset, i.e., the dataset $\{((x_i,y_i)) \}_{i=1}^n$ satisfies that $x_i \in \mathbb{L}^n, y_i \in \{1, -1\}, \forall i \in \{1, ..., n\}$,
\cite{cho2019large} first formulated the soft-margin Lorentz SVM as follows:
\begin{align*}
    \underset{w \in \mathbb{R}^{n + 1}}{\min} \; &\frac{1}{2} \norm{w}_{\mathbb{L}^n}^2 + C\sum_{i = 1}^n l_{\mathbb{L}^n}(-y_i [w, x_i])  \\
    s.t. \;\; & [w, w] > 0
\end{align*}
where $l_{\mathbb{L}^n}(z) = \max(0, \text{arcsinh}(1) - \text{arcsinh}(z))$. 
Unfortunately, this problem has a nonconvex constraint and nonconvex objective function.


\paragraph{PSVM}
Projected gradient descent is slow and may not converge to the global minimum. \cite{chien2021highly} removes the optimization constraints by first precomputing a reference point $p$ on the Poincar\'{e} ball and then projecting features via the logarithm map $\log_p$ onto the tangent space $T_p\D^n$, thereby casting hyperbolic SVM back into an ESVM. We call this method PSVM.
x
However, precomputing a reference point is not part of the SVM heuristics and may not find the maximum margin separating hyperplanes following this approach. The authors use a modified Graham Search \citep{graham1972efficient} to identify convex hulls on the Poincar\'{e} ball and use the midpoint of the two convex hulls as the reference point $p$. This does not have a valid justification when the two convex hulls overlap and may risk introducing a strong bias in the estimated hyperplane.

\paragraph{LSVMPP}
We build upon the framework of LSVM using our Euclidean parametrization of hyperplanes. Our reformulation lifts the nonconvex constraint $[w,w]>0$ and does not precompute any reference point. Using our parametrization described in the previous section, we have the following loss: 
\begin{align*}
    \underset{z \in \mathbb{R}^{n}, a \in \mathbb{R}}{\min} \; \frac{1}{2} \norm{z}^2 
    + C\sum_{i = 1}^n l_{\mathbb{L}^n}\left(y_i (\sinh(a)\norm{z}x_0  - \cosh(a) zx_r^\mathrm{T})\right).
\end{align*}
This lifts the constraint for optimization, but the objective still remains nonconvex. However, the empirical evaluation shows that this new SVM framework (LSVMPP) still outperforms the previous formulation LSVM, possibly by escaping the local minimums.

\subsection{Multiclass Classification Using SVM: Platt Scaling}
Suppose now that $y_i \in \{0, 1, ..., k\}$ for all $i \in \{1, ..., n\}$. A common way to train a multiclass classifier is by using the one-vs-all method, in which we train $k$ binary classifiers, using the $k$-th class as positive samples and negative samples otherwise. 

This method only gives us binary predictions if particular samples belong to a certain class, without weighting on the most and least probable class, which ultimately renders the predictions non-interpretable. To solve this issue, Platt scaling \citep{platt1999probabilistic} is a common method to transform binary values in the one-vs-all scheme in multiclass classification to the probabilities of each class. It itself is another machine learning optimization problem. We use ESVM as a demonstration in the following discussion.

In the $k$-th binary classifier, suppose $w^{(k)}$ is the trained vector determining a hyperplane (intercept included), the decision values are given by 
$$f^{(k)}(x_i) = \langle w^{(k)}, x_i \rangle.$$ 

This serves as a proxy of the distance between a sample to the separating hyperplane. Intuitively, if $f^{(k)}(x_i) > 0$ and if the sample is far away from the hyperplane, we should assign higher confidence, or probability, to the event that $x_i$ belong to the $k$-th class and lower probability for points closer to the hyperplane. \cite{platt1999probabilistic} argues that a logistic fit works well empirically to measure the associated probability, given by 
$$
P(y_i = 1 | f^{(k)}) = \frac{1}{1 + e^{Af^{(k)}(x_i) + B}}
$$
where $A, B$ are parameters trained by minimizing the negative log-likelihood of the training data through some optimization methods, typically through Newton's method.

This procedure repeats for every $f^{(k)}$. That is, we need $k$ (possibly) different sets of $(A, B)$ to transform all binary values into probabilities. 

\subsubsection{Adaptation of Platt Scaling on Hyperbolic Spaces}
\label{adaptation}
Poincar\'{e} SVM (PSVM) \citep{chien2021highly} easily adapts to Platt scaling since the hyperplane is trained on the tangent space of a reference point, which is Euclidean, and the decision values are similarly proxies for Euclidean distances (normed distance). 

This may become a problem for the other two hyperbolic models on the Lorentz manifold. In \citep{cho2019large}, the authors replace $f^{(k)}$ with 
$$
f^{(k), \text{LSVM}}(x) = [w^{(k)}, x]
$$
where $[\cdot, \cdot]$ denotes the Minkowski product. We argue that this loses its distance proxy functionality, and could be made more of a certain distance metric by 
$$
f^{(k), \text{LSVMPP}}(x) = \text{arcsinh}\left(-\frac{[w^{(k)}, x]}{\norm{w^{(k)}}_{\mathbb{L}^n}}\right)\norm{w^{(k)}}_{\mathbb{L}^n}
$$
and the last formulation is the input for the hyperbolic Platt scaling. 
Note that $\left|\text{arcsinh}\left(-\frac{[w^{(k)}, x]}{\norm{w^{(k)}}_{\mathbb{L}^n}}\right)\right|$ corresponds to the distance between the point $x$ and the hyperplane determined by $w^{(k)}$ (see \citep[Lemma 4.1]{cho2019large}\footnote{There is a slight imprecision in \citep[Lemma 4.1]{cho2019large} that the absolute value sign was missing.}). So the above formula denotes a normed signed distance between the point $x$ and the hyperplane determined by $w^{(k)}$.
We report that empirically there is an edge in accuracy using this formulation instead of the previous one (see \Cref{tab:real_dataset_result}). However, it remains to show whether the logistic fit is still a valid assumption, particularly for datasets with a latent hierarchical structure on hyperbolic spaces.

\subsection{Hyperbolic SVM Implementation and Hyperparameters}\label{implementation and hyperparams}
We rely on scikit-learn \citep{scikit-learn} for the realization of Euclidean and Poincar\'{e} SVM. The relevance of Poincar\'{e} SVM training is that after identifying a reference point and projecting Poincar\'{e} features onto its tangent space, features are Euclidean and hence Euclidean SVM training applies, although its intercept should be set to 0, since the precomputed reference point already encodes the intercept information. For Lorentz SVM, we inherit from the code provided in \citep{cho2019large,chien2021highly} and modify it for better run-time performances; and for our model LSVMPP, we use PyTorch \citep{paszke2017automatic} to help to compute the gradient automatically. 

Hyperparameters for four different SVM models should be different since they have different magnitudes in both loss and penalty terms. Margins, in particular, in the Euclidean, Poincar\'{e}, and Lorentz spaces for the same set of data are different. Therefore, hyperparameter search should be made separate for each model. 

The best performances of the Euclidean and Poincar\'{e} SVM are both using $C = 5$, with a learning rate of 0.001 and 3000 epochs. An early stop mechanism is built in the scikit-learn implementation \citep{scikit-learn} of Euclidean SVM, so not all 3000 epochs are utilized, and empirically 500 epochs suffice the respective optimality for all tested data; the best performance of LSVM and LSVMPP are in general brought by $C = 0.5$, with a learning rate around $10^{-10}$ (depending on the initial scale of the dataset) with 500 epochs. We note that the abnormally small learning rate is explained by the necessity of escaping the local minimum of a complex non-convex objective function.

\subsection{Details of Synthetic Data Generation}
\label{app:svm_syn_data_gen}
This section provides more detail on generation of synthetic datasets. As noted previously, we have two types of synthetic datasets: Gaussian mixtures and explicit trees. Gaussian mixture datasets are named sim\_data\_1, sim\_data\_2, and sim\_data\_3 and explicit trees are named sim\_data\_4, sim\_data\_5, sim\_data\_6. All synthetic datasets are visualized in \Cref{fig:all_svm_data}. 


In particular, to generate an explicit tree dataset, we simulate a birth-and-death process. In this process, the distribution of offspring follows a Poisson distribution with a parameter value of $\lambda = 3$, with an additional 1 added (resulting in a Poisson random variable starting with 1). The three plotted explicit trees represent three different realizations of this process. Once the tree structure is established, we initialize Euclidean embeddings that closely resemble the final embeddings in \Cref{fig:all_svm_data}. We then apply the tree embedding method, as described in \Cref{sec:epl}, utilizing Euclidean parametrization to optimize the embedding. We train with SGD using a learning rate of 10 for 50000 epochs.



\subsection{Full Results}
\label{app:svm_full_results}


We visualize all datasets on the Poincar\'{e} ball in \Cref{fig:all_svm_data}.

The average accuracy and F1 scores of five different runs on the same train-test split of each dataset are reported below in \Cref{tab:syn_dataset_result} and  \Cref{tab:real_dataset_result}. 

To maintain consistency in the evaluation, we adopt the train-test split provided in the GitHub referenced in \cite{chien2021highly} for the olsson, CIFAR, and fashion-MNIST datasets. For all other datasets, we utilize a 75\%/25\% train-test split stratified based on the class assignments.


In addition, we also compare to a model using Platt scaling on raw output, named ``LSVMPP (raw)" in the table below. ``LSVMPP (arcsinh)" is the version with adapted input for Platt scaling. The two versions share the same learning rate. See \Cref{adaptation} for more details. Notice that LSVMPP (raw) performs better than the adapted version when Euclidean SVM is at the top-notch; otherwise, it is worse and less stable in performance than the adapted version, indicating a lesser fit of the Platt training step with respect to the raw decision values.

\begin{figure}[htbp!]
    \centering
    
    \begin{subfigure}
        \centering
        \includegraphics[width=0.3\textwidth]{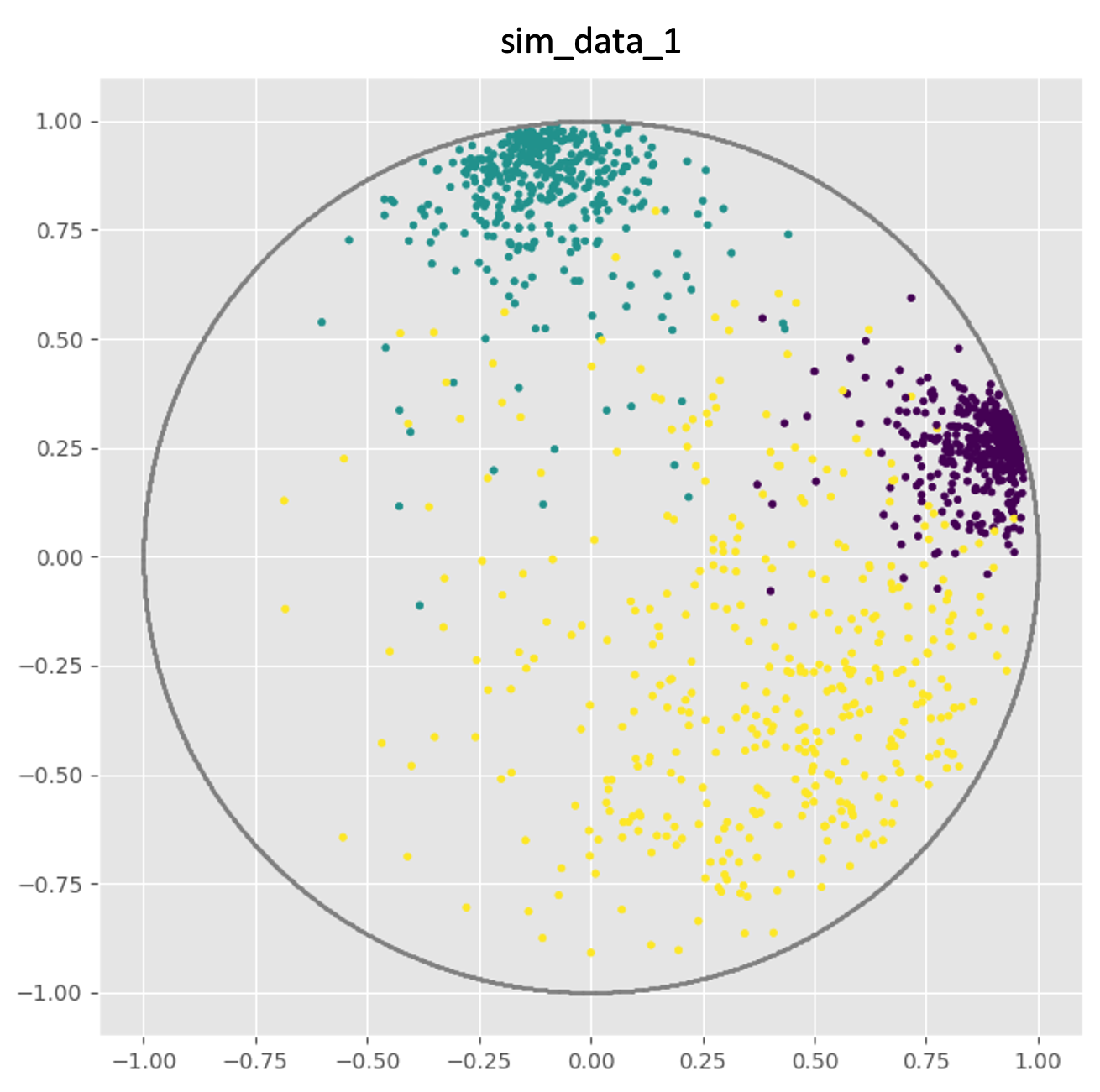}
    \end{subfigure} ~ \begin{subfigure}
        \centering
        \includegraphics[width=0.3\textwidth]{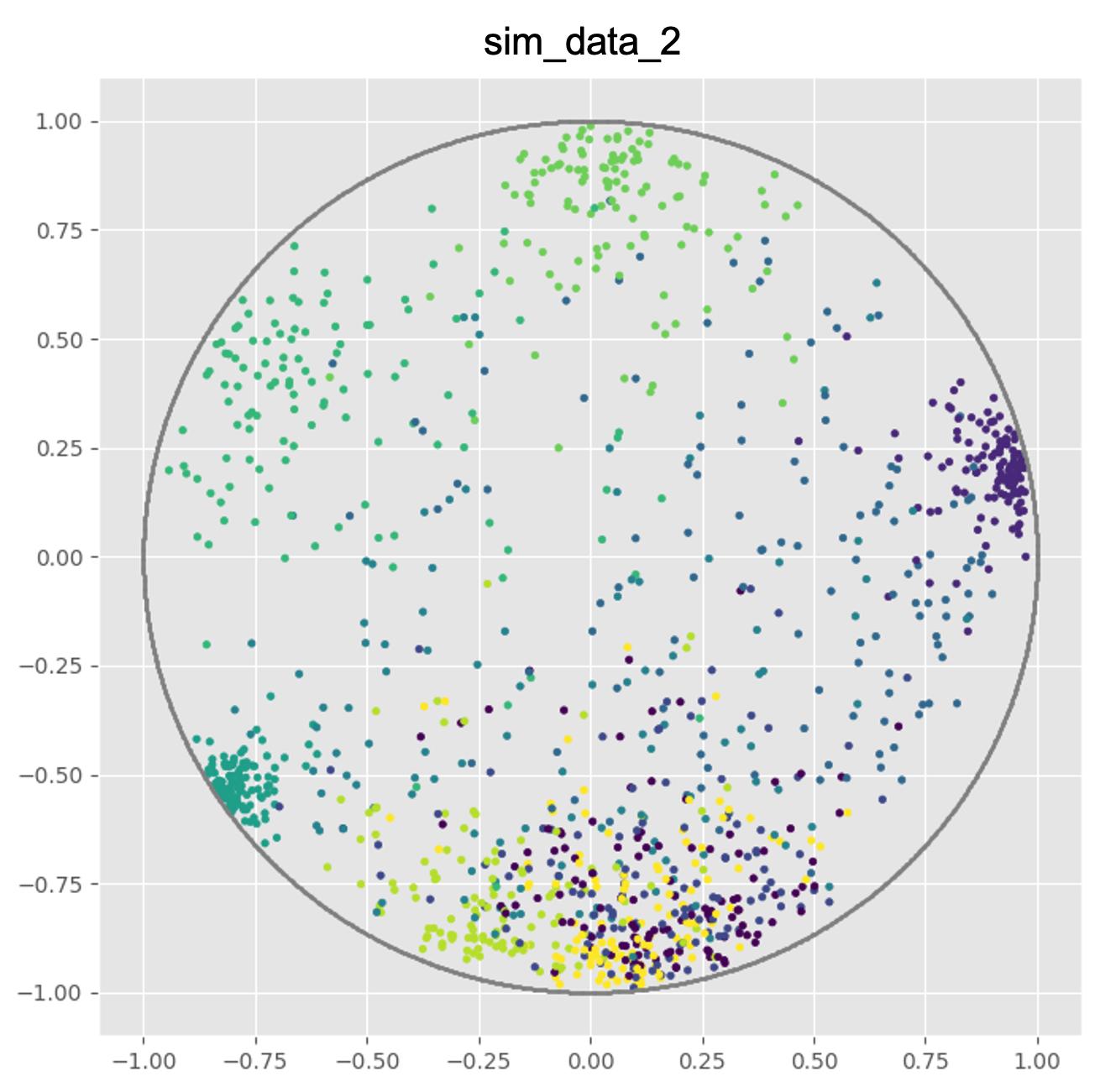}
    \end{subfigure} ~ 
    \begin{subfigure}
        \centering
        \includegraphics[width=0.3\textwidth]{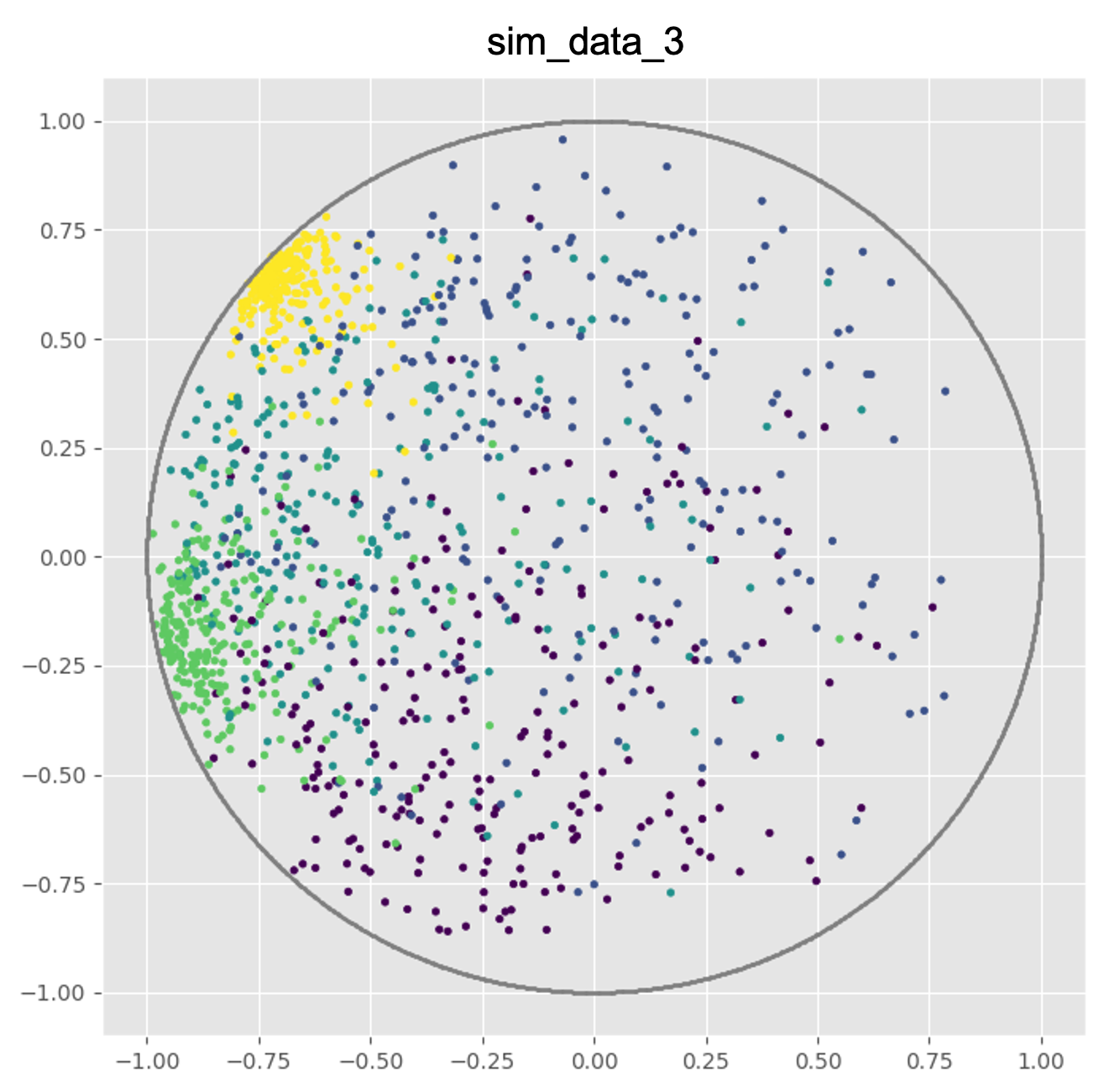}
    \end{subfigure} \\

    \begin{subfigure}
        \centering
        \includegraphics[width=0.3\textwidth]{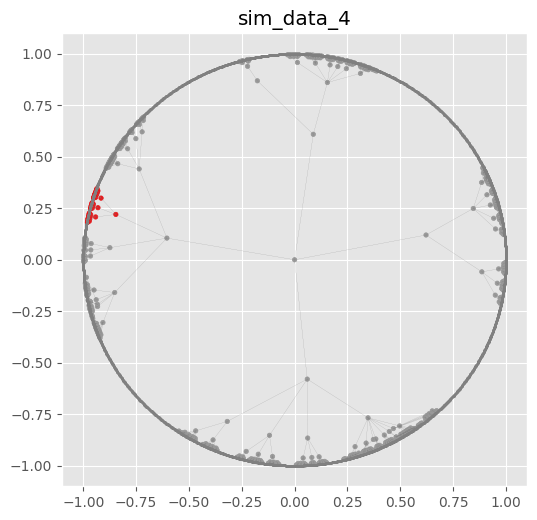}
    \end{subfigure} ~ \begin{subfigure}
        \centering
        \includegraphics[width=0.3\textwidth]{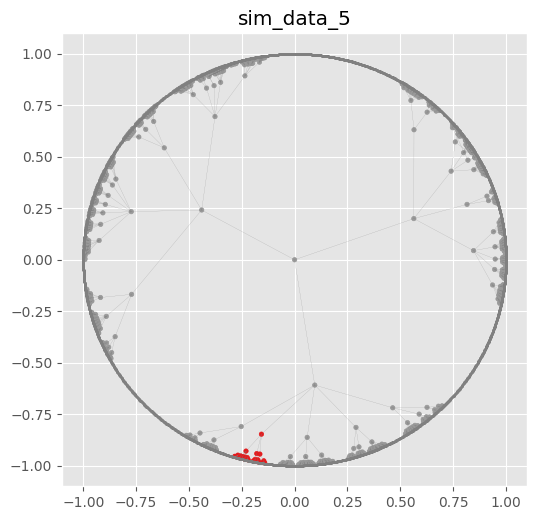}
    \end{subfigure} ~ 
    \begin{subfigure}
        \centering
        \includegraphics[width=0.3\textwidth]{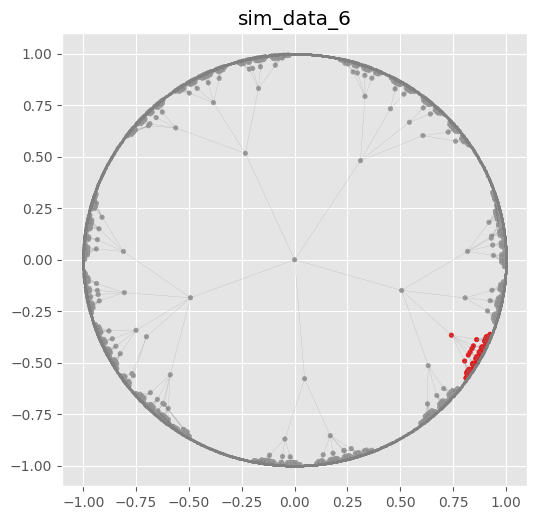}
    \end{subfigure} \\

    \begin{subfigure}
        \centering
        \includegraphics[width=0.3\textwidth]{images/cifar.png}
    \end{subfigure} ~ 
    \begin{subfigure}
        \centering
        \includegraphics[width=0.3\textwidth]{images/fashion-mnist.png}
    \end{subfigure} ~ \begin{subfigure}
        \centering
        \includegraphics[width=0.3\textwidth]{images/paul.png}
    \end{subfigure}  \\

    \begin{subfigure}
        \centering
        \includegraphics[width=0.3\textwidth]{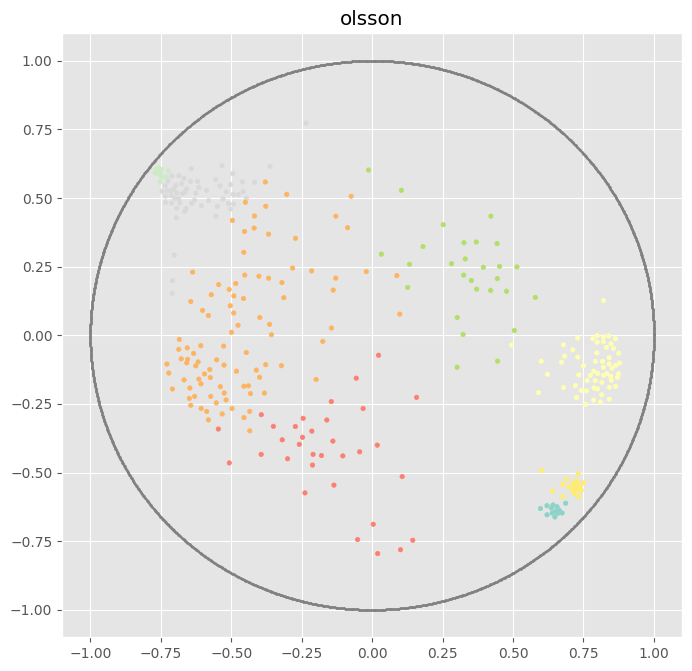}
    \end{subfigure} ~ 
    \begin{subfigure}
        \centering
        \includegraphics[width=0.3\textwidth]{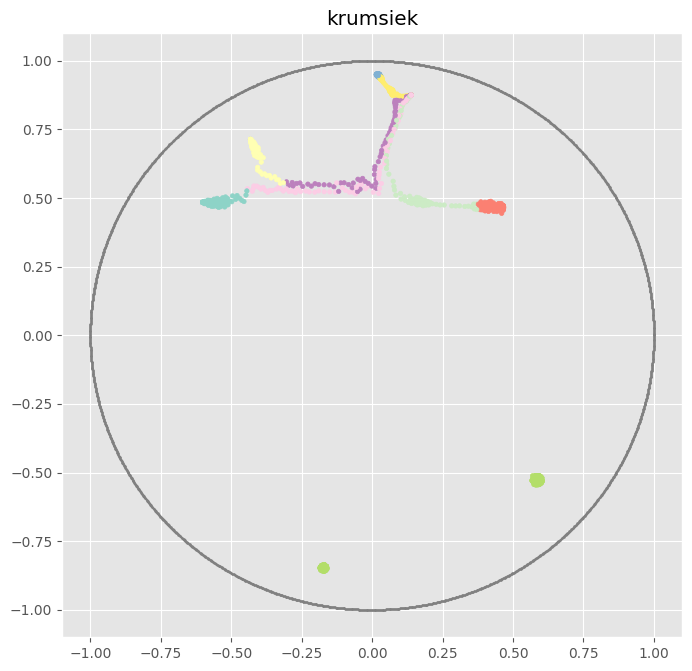}
    \end{subfigure} ~ \begin{subfigure}
        \centering
        \includegraphics[width=0.3\textwidth]{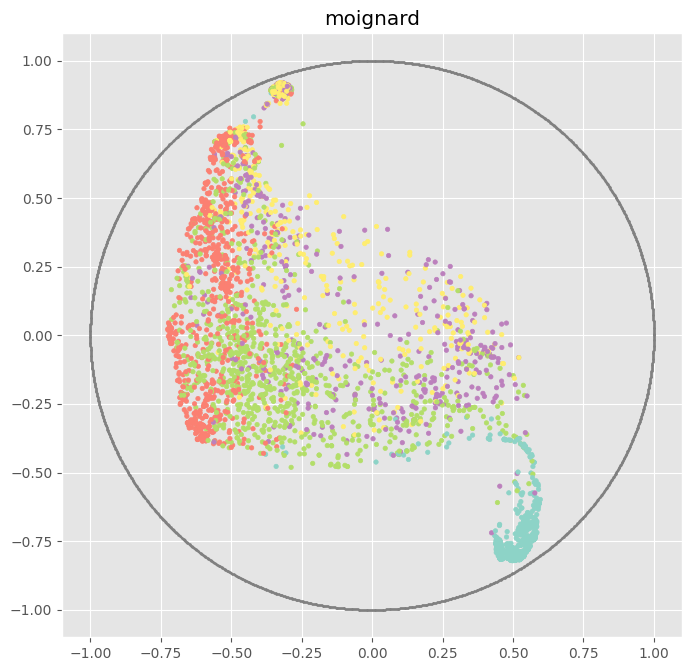}
    \end{subfigure}  \\
    \caption{Simulated (first two rows, with $k = 3, k = 5, k = 10$ from left to right for the first row respectively and $k=2$ for the second row) and real (last two rows) datasets: CIFAR ($k=10$), fashion-MNIST ($k=10$), paul ($k=19$), olsson ($k=8$), krumsiek ($k=11$), and moignard $k=7$. In each plot, different colors represent different classes. For the explicit tree datasets we also include the tree edges.}
    \label{fig:all_svm_data}
\end{figure}


\begin{table}[htbp!]
    \centering
    \caption{Mean accuracy and macro F1 score of synthetic Dataset in \Cref{fig:all_svm_data}.}
    \resizebox{0.65\columnwidth}{!}{%
    \begin{tabular}{cc|cc} \toprule
         & algorithm & accuracy (\%) & F1 \\ 
         \midrule
        \multirow{4}{*}{sim 1} & ESVM & 94.67 $\pm$ 0.00 & 0.9465 $\pm$ 0.00 \\
         & LSVM &  94.67 $\pm$ 0.00 & 0.9465 $\pm$ 0.00 \\
          & PSVM & 91.66 $\pm$ 0.00 & 0.9151 $\pm$ 0.00 \\
          & LSVMPP (raw) & 94.67 $\pm$ 0.00 & 0.9464 $\pm$ 0.00 \\
          & LSVMPP (arcsinh) & 94.67 $\pm$ 0.00 & 0.9464 $\pm$ 0.00 \\
        \midrule 
        \multirow{4}{*}{sim 2} & ESVM & 68.67 $\pm$ 0.00 & 0.6656 $\pm$ 0.00 \\ 
         & LSVM & 66.66 $\pm$ 0.00 & 0.6336 $\pm$ 0.00 \\
         & PSVM & 64.66 $\pm$ 0.00 & 0.6143 $\pm$ 0.00 \\
         & LSVMPP (raw) & \textbf{70.33 $\pm$ 0.00} & \textbf{0.6920 $\pm$ 0.00}	\\
         & LSVMPP (arcsinh) & 69.00 $\pm$ 0.00 & 0.6840 $\pm$ 0.00 \\
         \midrule
        \multirow{4}{*}{sim 3} & ESVM & 61.67 $\pm$ 0.00 & \textbf{0.5908 $\pm$ 0.00} \\ 
         & LSVM & 39.00 $\pm$ 0.00 & 0.3236 $\pm$ 0.00 \\ 
         & PSVM & 61.00 $\pm$ 0.00 & 0.5606 $\pm$ 0.00 \\ 
         & LSVMPP (raw) & 60.67 $\pm$ 0.00	& 0.5468 $\pm$ 0.00 \\
         & LSVMPP (arcsinh) & \textbf{61.93 $\pm$ 0.01} & 0.5881 $\pm$ 0.01 \\
         \midrule
         \multirow{4}{*}{sim 4} & ESVM & 93.40 $\pm$ 0.00 & 0.0000 $\pm$ 0.00 \\ 
         & LSVM & 96.04 $\pm$ 3.24 & 0.4000 $\pm$ 0.49 \\ 
         & PSVM & 99.53 $\pm$ 0.00 & 0.9630 $\pm$ 0.00 \\ 
         & LSVMPP (raw) & \textbf{100.00 $\pm$ 0.00}	& \textbf{1.0000 $\pm$ 0.00} \\
         & LSVMPP (arcsinh) & \textbf{100.00 $\pm$ 0.00} & \textbf{1.0000 $\pm$ 0.00} \\
         \midrule
         \multirow{4}{*}{sim 5} & ESVM & 96.25 $\pm$ 0.00 & 0.0000 $\pm$ 0.00 \\ 
         & LSVM & 96.25 $\pm$ 0.00 & 0.0000 $\pm$ 0.00 \\ 
         & PSVM & 99.63 $\pm$ 0.00 & 0.9474 $\pm$ 0.00 \\ 
         & LSVMPP (raw) & 99.63 $\pm$ 0.00 & 0.9474 $\pm$ 0.00 \\ 
         & LSVMPP (arcsinh) & {99.63 $\pm$ 0.00} & {0.9474 $\pm$ 0.00} \\ 
         \midrule
         \multirow{4}{*}{sim 6} & ESVM & 97.64 $\pm$ 0.00 & 0.7805 $\pm$ 0.00 \\ 
         & LSVM & 99.63 $\pm$ 0.13 & 0.9711 $\pm$ 0.01 \\ 
         & PSVM & 97.91 $\pm$ 0.00 & 0.8095 $\pm$ 0.00 \\ 
         & LSVMPP (raw) & 99.53 $\pm$ 0.42	& 0.9648 $\pm$ 0.03 \\
         & LSVMPP (arcsinh) & \textbf{99.63 $\pm$ 0.13} & \textbf{0.9714 $\pm$ 0.01} \\
         \bottomrule
    \end{tabular}%
    }
    \label{tab:syn_dataset_result}
\end{table}

\begin{table}[htbp!]
    \centering
    \caption{Mean accuracy and macro F1 score of real Dataset in \Cref{fig:all_svm_data}.}
    \resizebox{0.65\columnwidth}{!}{%
    \begin{tabular}{cc|cc} \toprule
    & algorithm & accuracy (\%) & F1 \\ 
    \midrule
        \multirow{4}{*}{CIFAR-10} & ESVM & 91.88 $\pm$ 0.00 & 0.9191 $\pm$ 0.00 \\ 
         & LSVM & 91.88 $\pm$ 0.00 & 0.9189 $\pm$ 0.00 \\
         & PSVM & 91.81 $\pm$ 0.00 & 0.9182 $\pm$ 0.00 \\
         & LSVMPP (raw) & 91.94 $\pm$ 0.00 & 0.9195 $\pm$ 0.00 \\
         & LSVMPP (arcsinh) & \textbf{91.96 $\pm$ 0.00} & \textbf{0.9197 $\pm$ 0.00} \\
         \midrule
        \multirow{4}{*}{fashion-MNIST} & ESVM & 86.37 $\pm$ 0.00 & 0.8665 $\pm$ 0.00 \\ 
         & LSVM & 71.59 $\pm$ 0.07 & 0.6588 $\pm$ 0.08 \\ 
         & PSVM & 86.57 $\pm$ 0.00 & 0.8665 $\pm$ 0.00 \\ 
         & LSVMPP (raw) & 89.35 $\pm$ 0.00 & 0.8939 $\pm$ 0.00 \\
         & LSVMPP (arcsinh) & \textbf{89.49 $\pm$ 0.00} & \textbf{0.8955 $\pm$ 0.00} \\
         \midrule 
         \multirow{4}{*}{paul} & ESVM & 55.05 $\pm$ 0.00 & 0.4073 $\pm$ 0.00 \\ 
         & LSVM & 58.36 $\pm$ 0.07 & 0.4517 $\pm$ 0.00 \\ 
         & PSVM & 55.25 $\pm$ 0.00 & 0.3802 $\pm$ 0.00 \\ 
         & LSVMPP (raw) & 62.55 $\pm$ 0.14 & 	0.4579 $\pm$ 0.01 \\
         & LSVMPP (arcsinh) & \textbf{62.64 $\pm$ 0.05} & \textbf{0.5024 $\pm$ 0.00} \\
        \midrule
         \multirow{4}{*}{olsson} & ESVM & 72.72 $\pm$ 0.00 & 0.4922 $\pm$ 0.00 \\ 
         & LSVM & 81.82 $\pm$ 0.00 & 0.7542 $\pm$ 0.00 \\
         & PSVM & \textbf{88.63 $\pm$ 0.00} & \textbf{0.8793 $\pm$ 0.00} \\
         & LSVMPP (raw) & 80.91 $\pm$ 0.45 & 0.7142 $\pm$ 0.03 \\
         & LSVMPP (arcsinh) & 84.09 $\pm$ 0.00 & 0.8429 $\pm$ 0.00 \\
         \midrule
        \multirow{4}{*}{krumsiek} & ESVM & 82.19 $\pm$ 0.00 & 0.6770 $\pm$ 0.00 \\ 
         & LSVM & 85.62 $\pm$ 0.39 & 0.6933 $\pm$ 0.92 \\ 
         & PSVM & 84.06 $\pm$ 0.00 & 0.6908 $\pm$ 0.00 \\ 
         & LSVMPP (raw) & 83.75 $\pm$ 0.28 & 0.6403 $\pm$ 0.01 \\
         & LSVMPP (arcsinh) & \textbf{86.25 $\pm$ 0.00} & \textbf{0.7079 $\pm$ 0.00} \\
         \midrule 
         \multirow{4}{*}{moignard} & ESVM & \textbf{67.78 $\pm$ 0.00} & \textbf{0.5934 $\pm$ 0.00} \\ 
         & LSVM & 64.88 $\pm$ 0.38 & 0.5502 $\pm$ 0.22 \\ 
         & PSVM & 60.77 $\pm$ 0.00 & 0.5167 $\pm$ 0.00 \\ 
         & LSVMPP (raw) & 65.77 $\pm$	0.13 &	0.5671	$\pm$ 0.00 \\
         & LSVMPP (arcsinh) & 65.22 $\pm$ 0.63 & 0.5719 $\pm$ 0.02 \\
         \bottomrule
    \end{tabular}%
    }
    \label{tab:real_dataset_result}
\end{table}


\end{document}